\documentclass[]{interact}

\usepackage{epstopdf}
\usepackage[caption=false]{subfig}

\usepackage{geometry}
\usepackage{marginnote}

\usepackage[longnamesfirst,sort]{natbib}
\bibpunct[, ]{(}{)}{;}{a}{,}{,}

\usepackage{times}
\usepackage{url}
\usepackage{xcolor}
\usepackage{soul}
\usepackage[utf8]{inputenc}
\usepackage[small]{caption}
\usepackage{amssymb}
\usepackage{algorithm}
\usepackage[noend]{algpseudocode}

\newcommand{\tip}{{\bf T}}
\newcommand{\alc}{\mathcal{ALC}}
\newcommand{\alct}{\mathcal{ALC}+\tip_{\bf R}}

\newcommand {\allworlds} {\mathcal{W}_{\mathcal{K}}}
\newcommand {\kk} {\mathcal{K}}
\newcommand{\RR}{{\mathcal{R}}}
\newcommand{\TT}{{\mathcal{T}}}
\newcommand{\AAA}{{\mathcal{A}}}
\newcommand{\pfl}{\tip^{\textsf{\tiny CL}}}
\newcommand{\II}{{\mathcal{I}}}
\newcommand{\emme}{{\mathcal{M}}}
\newcommand{\PP}{{\mathbb{P}}}

\theoremstyle{plain}
\newtheorem{theorem}{Theorem}[section]

\theoremstyle{definition}
\newtheorem{definition}[theorem]{Definition}
\newtheorem{example}[theorem]{Example}

\theoremstyle{remark}

\sethlcolor{yellow}

\begin{document}


\title{A Description Logic Framework for Commonsense Conceptual \\ Combination Integrating Typicality, Probabilities \\ and Cognitive Heuristics}

\author{
\name{Antonio Lieto\textsuperscript{a}\thanks{CONTACT Antonio Lieto. Email: antonio.lieto@unito.it} and Gian Luca Pozzato\textsuperscript{b}\thanks{CONTACT Gian Luca Pozzato. Email: gianluca.pozzato@unito.it}}
\affil{\textsuperscript{a}Dipartimento di Informatica, Università di Torino, Italy and ICAR-CNR, Istituto di Calcolo e Reti ad Alte Prestazioni, Palermo, Italy; \\ \textsuperscript{b}Dipartimento di Informatica, Università di Torino, Italy}
}

\maketitle

\begin{abstract}
We propose a nonmonotonic Description Logic of typicality able to account for the phenomenon of the combination of prototypical concepts.
The proposed logic relies on the logic of typicality $\alct$, whose semantics is based on the notion of rational closure, as well as on the distributed semantics  of probabilistic Description Logics, and is equipped with a cognitive heuristic used by humans for concept composition. 

We first extend the  logic of typicality $\alct$ by typicality inclusions of the form $p \ :: \ \tip(C) \sqsubseteq D$, whose intuitive meaning is that ``we believe with degree $p$ about the fact that typical $C$s are $D$s''. As in the distributed semantics, we  define different scenarios containing only some typicality inclusions, each one having a suitable probability. We then exploit such scenarios in order to ascribe  typical properties to a concept $C$ obtained as the combination of two prototypical concepts.
We also show that reasoning in the proposed Description Logic is \textsc{ExpTime}-complete as for the underlying standard Description Logic $\alc$.
\end{abstract}


\section{Introduction}
Inventing novel concepts by combining the typical knowledge of pre-existing ones is among the most creative cognitive abilities exhibited by humans. This generative phenomenon highlights some crucial aspects of the knowledge processing capabilities in human cognition and concerns high-level capacities associated to creative thinking and problem solving. Still, it represents an open challenge in the field od Artificial Intelligence (AI) [\cite{boden1998creativity, colton2012computational}]. 
Dealing with this problem requires, from an AI perspective, the harmonization of two conflicting requirements that are hardly accommodated in symbolic systems (including formal ontologies [\cite{frixione2011representing}]): the need of a syntactic and semantic compositionality (typical of logical systems) and that one concerning the exhibition of typicality effects.  
According to a well-known argument, in fact, prototypes are not compositional [\cite{fodor1981present, osherson1981adequacy, murphy2004big, gleitman2012can}]. The argument runs as follows: consider a concept like \emph{pet fish}. It results from the composition of the concept \emph{pet} and of the concept \emph{fish}. However, the prototype of \emph{pet fish} cannot result from the composition of the prototypes of a pet and a fish: e.g. a typical pet is furry and warm, a typical fish is grayish, but a typical pet fish is neither furry and warm nor grayish (typically, it is red).


In this work, we provide a framework able to account for this type of human-like concept combination by explicitly relying on a formalization of the prototype theory \footnote{Thus providing additional arguments in favor of the theoretical positions according to which it is possible to reconcile compositionality and prototypes [\cite{prinz2012regaining,hampton2000concepts}].}.
We propose a nonmonotonic Description Logic (from now on DL) of typicality called $\pfl$ (typicality-based compositional logic). 
This logic combines two main ingredients. 
The first one relies on the DL of typicality $\alct$ introduced in [\cite{AIJ2014}].
In this logic, ``typical'' properties can be directly specified  by means of a
``typicality'' operator $\tip$ enriching  the underlying DL, and a TBox can contain inclusions of the form $\tip(C) \sqsubseteq D$ to represent that ``typical $C$s are also $Ds$''. As a difference with standard DLs, in the logic $\alct$  one can consistently express exceptions and reason about defeasible inheritance as well. 
For instance, a knowledge base can consistently express that ``normally, athletes are in fit'', whereas ``sumo wrestlers usually are not in fit'' by the typicality inclusions 
\begin{quote}
  $\tip (\mathit{Athlete}) \sqsubseteq \mathit{Fit}$  \\
  $\tip (\mathit{SumoWrestler}) \sqsubseteq  \lnot \mathit{Fit},$
\end{quote}
given that $\mathit{SumoWrestler} \sqsubseteq \mathit{Athlete}$. 
The semantics of the $\tip$ operator is  characterized by the properties of \emph{rational logic} [\cite{whatdoes}], recognized as the core properties of nonmonotonic reasoning. $\alct$ is characterized by a minimal model semantics corresponding to an extension to DLs of a notion of \emph{rational closure} as defined in [\cite{whatdoes}] for propositional logic: the
idea is to adopt a preference relation among $\alct$ models, where intuitively a model is preferred to another one if it contains less exceptional elements, as well as a notion of \emph{minimal entailment} restricted to models that are minimal with respect to such preference relation.  
As a consequence, $\tip$ inherits well-established properties like \emph{specificity} and \emph{irrelevance}: in the example, the logic $\alct$ allows us to infer $\tip(\mathit{Athlete} \sqcap \mathit{Bald}) \sqsubseteq \mathit{Fit}$ (being bald is irrelevant with respect to being in fit) and, if one knows that Hiroyuki is a  typical sumo wrestler, to infer that he is not in fit, giving preference to the most specific information.

As a second ingredient, we rely on a distributed semantics similar to the DISPONTE semantics proposed by [\cite{disponteijcai,riguzzi}] for probabilistic extensions of DLs, more precisely we extend the typicality inclusions of $\alct$ with probabilities intended as \emph{degrees of belief}. As in the DISPONTE semantics, we make use of these degrees in order to define different scenarios, corresponding in our case to subsets of typicality inclusions.
  More in detail, we restrict our adoption of DISPONTE to label inclusions $\tip(C) \sqsubseteq D$ with a real number between 0.5 and 1, representing the degree of belief in such an inclusion\footnote{We want to stress that, as in any probabilistic formal framework, probabilities are assumed to come from an application domain. This is true also for other frameworks such as, for example, fuzzy logics or  probabilistic extensions of logic programs. In this paper, we focus on the proposal of the formalism itself, therefore the machinery for obtaining probabilities from a dataset of the application domain is out of the scope.}, assuming that each axiom is independent from each other (as in the DISPONTE semantics).  The resulting knowledge base defines a probability distribution over \emph{scenarios}: roughly speaking, a scenario is obtained by choosing, for each typicality inclusion, whether it is considered as true or false. As a consequence, two types of vagueness are expressed by a typicality inclusion $p \ :: \ \tip(C) \sqsubseteq D$: on the one hand, the typicality operator $\tip$ allows one to express that only typical members of $C$ are also members of $D$, on the other hand, $p$ allows one to have typical properties having different degrees of belief in each typicality inclusions. In a slight extension of the above example, we could have the need of representing that both the typicality inclusions about athletes and sumo wrestlers have a probability of  $80\%$, whereas we also believe that athletes  are usually young with a higher probability of $95\%$, with the following KB:

\begin{enumerate}
\item $\mathit{SumoWrestler} \sqsubseteq \mathit{Athlete}$ 
\item $0.8 \ :: \ \tip (\mathit{Athlete}) \sqsubseteq \mathit{Fit}$ 
\item $0.8 \ :: \ \tip (\mathit{SumoWrestler}) \sqsubseteq  \lnot \mathit{Fit}$
\item $0.95 \ :: \ \tip (\mathit{Athlete}) \sqsubseteq \mathit{YoungPerson}$
\end{enumerate}

We consider eight different scenarios,
representing all possible combinations of typicality inclusion: as an example, we can consider the one in which (2) and (4) hold, whereas (3) does not. In this scenario, we have, for instance, that also sumo wrestlers are  fit. We  equip each scenario with a degree depending on those of the involved typicality inclusions. In the above example, the scenario containing only (2) and (4) has a probability $15,2\%$, obtained by the product of $0.8$ (from 2), $1-0.8$ (since we do not include 3) and $0.95$ (from 4).


As an additional element of the proposed formalization we employ a method inspired by cognitive semantics [\cite{kamp1995prototype,smith1988combining,hampton1987inheritance,hampton1988overextension}] for the identification of a dominance effect between the concepts to be combined. Namely, for every combination, we distinguish a  HEAD and a MODIFIER, where the HEAD represents the stronger element of the combination.  The basic idea is as follows: given a KB and two concepts $C_H$ and $C_M$ occurring in it, where $C_H$ is the HEAD and $C_M$ the MODIFIER, we consider only {\em some}  scenarios in order to define a revised knowledge base, enriched by typical properties of the combined concept $C \sqsubseteq C_H \sqcap C_M$. Such scenarios are those (i) consistent with respect to the initial knowledge base, (ii) not \emph{trivial}, i.e.  we discard those with the highest probability, containing either \emph{all} properties that can be consistently ascribed to $C$ or \emph{all} properties of the HEAD that can be consistently ascribed to $C$, and (iii) giving preference to the typical properties of the HEAD $C_H$ (with respect to those of $C_M$) having the highest probability. 

 We are able to exploit the logic $\pfl$ in two different perspectives. On the one hand, we show that it is able to capture well established examples in the literature of cognitive science concerning concept combination and, as such, we argue that $\pfl$ is a promising candidate to tackle the problem of typicality-based concept combination (Sections \ref{sezione pet fish}, \ref{sezione linda} and \ref{sezione metafora}).
On the other hand, we use $\pfl$ as a tool for the generation and the exploration of novel creative concepts (Section \ref{sez:novel}), that could be useful in many applicative scenarios, ranging from video games to the creation of new movie or story characters. 

As a further result, we show that the proposed approach is essentially inexpensive, in the sense that reasoning in  $\pfl$ is \textsc{ExpTime}-complete as for the underlying standard $\alc$ Description Logic.

The plan of the paper is as follows. In Section \ref{sez:background} we recall the two semantics being the starting points of our proposal, namely the rational closure for the DL of typicality $\alct$ and the DISPONTE semantics for probabilistic DLs. In Section \ref{sez:pfl} we present the logic $\pfl$ for concept combination, and we show its reasoning complexity. In Section \ref{sez:esempi} we show that the proposed logic $\pfl$ is able to capture some well known and paradigmatic examples of concept combination coming from the  cognitive science literature. In Section \ref{sez:novel} we exploit the logic $\pfl$ in the application domain of computational creativity: i.e. we show how  $\pfl$ can be used for inventing/generating novel concepts as the result of the combination of two (or more) prototypes. In Section \ref{sez:iterativa} we proceed further by showing that the logic $\pfl$ can be iteratively applied to combine prototypical concepts already resulting from the combination of prototypes. 
 We conclude in Section \ref{sez:conclusioni} by mentioning some related approaches addressing the problem of common-sense concept combination, as well as by discussing on possible future works.

\section{Background: DLs of typicality and probabilistic DLs}\label{sez:background}
As introduced above, the main aim of this work is to introduce a nonmonotonic Description Logic able to deal with the combination of prototypical concepts. In order to achieve this goal, we exploit two well established logical frameworks:
\begin{itemize}
   \item the nonmonotonic logic of typicality $\alct$  based on a notion of rational closure for DLs.
   \item the DISPONTE semantics of probabilistic extensions of DLs
\end{itemize}
In this section we briefly recall such ingredients as well as Description Logics themselves, before introducing our proposal in Section \ref{sez:pfl}.

\subsection{Description Logics}

The family of Description Logics is one of the most important formalisms of knowledge representation. DLs are reminiscent of the early semantic networks and of frame-based systems (for a detailed introduction see [\cite{baader2003description}]. They offer two key advantages: on the one hand, they have a well-defined semantics based on first-order logic and, on the other hand, they offer a good
trade-off between expressivity of the language and computational complexity. DLs have been
successfully implemented by  a range of systems and they are at the
base of languages for the semantic web such as OWL (Ontology Web Language).


A  DL knowledge base (KB) comprises two components: 

\begin{itemize}
    \item the TBox,  containing the
definition of concepts (and possibly roles), and a specification of
inclusions relations among them;
\item the ABox containing
instances of concepts and roles, in other words,  properties and
relations of individuals.
\end{itemize}

\noindent As an example, consider the following knowledge base, whose TBox contains the inclusion relations:

\begin{quote}
    $\mathit{PhDStudent} \sqsubseteq \mathit{Student}$ \hfill (1) \\
    $\mathit{Student} \sqcap \mathit{Worker} \sqsubseteq \mathit{TaxPayer}$ \hfill (2) \\
    $\mathit{TaxPayer} \sqsubseteq \exists \mathit{hasId}.(\mathit{Code} \sqcup \mathit{NationalInsuranceNumber}) \hfill (3) $
\end{quote}

\noindent and whose ABox contains:
\begin{quote}
    $\mathit{PhDStudent}(\mathit{max})$ \hfill (4) \\
    $\mathit{Worker}(\mathit{max})$ \hfill (5) \\
    $\mathit{NationalInsuranceNumber}(12345)$ \hfill (6) \\
    $\mathit{hasId}(\mathit{max},12345) \hfill (7) $
\end{quote}

\noindent represents that PhD students are students (1), and that working students pay working taxes (2). Moreover, we have that tax payers possess an ID which is either a code or a NIN (National Insurance Number) (3). $\mathit{Student}$, $\mathit{PhDStudent}$, $\mathit{Worker}$, $\mathit{Code}$, $\mathit{TaxPayer}$, and $\mathit{NationalInsuranceNumber}$ are \emph{concepts}. $\mathit{hasId}$ is a \emph{role}, defining a relation between tax payers and IDs (codes and NINs). Furthermore, we have that Max is a both a student (4) and a worker (5), and he possesses the ID $12345$ (7) which is a NIN (6).

The basic Description Logic is called $\alc$. Concepts in $\alc$ are built from an alphabet of concept names $\mathtt{C}$ and of role names $\mathtt{R}$. Concept names define atomic concepts. Atomic concepts and role names are then combined in order to build complex concepts: $C \sqcap D$ (intersection), $C \sqcup D$ (union), $\lnot C$ (complement), $\forall R.C$ (universal qualified restriction) and $\exists R.C$ (existential qualified restriction). Extensions as well as fragments of $\alc$ have been studied by considering other sets of operators. 

Models in DLs are Kripke structures of the form $$\langle \Delta^\II, .^\II \rangle$$ where $\Delta^\II$ is a non empty set of elements called the \emph{domain}, whereas $.^\II$ is the \emph{extension function} mapping each atomic concept $C$ to the set of domain elements $C^\II$ being members of such concept, and mapping each atomic role $R$ to pairs of domain elements $R^\II$ being related by such role. For instance, $\mathit{Student}^\II$ is the set of students of belonging to the domain under consideration. The function $.^\II$ is then extended to complex concepts in order to provide a suitable semantics to each connective/operator, that is to say:
\begin{itemize}
    \item $(\lnot C)^\II = \Delta^\II \backslash C^\II$, i.e. the extension of $\lnot C$ is the complement of the extension of $C$, in other words it contains elements of $\Delta^\II$ not being $C$ elements;
    \item $(C \sqcap D)^\II=C^\II \cap D^\II$, for instance the set of working students $\mathit{Student} \sqcap \mathit{Worker}$ corresponds to the intersection of the extension of $\mathit{Student}$ and the extension of $\mathit{Worker}$;
    \item $(C \sqcup D)^\II=C^\II \cup D^\II$, for instance the extension of  $\mathit{Male} \sqcup \mathit{Female}$ corresponds to the union of the extension of $\mathit{Male}$ and the extension of $\mathit{Female}$;
    \item $(\exists R.C)^\II = \{x \in \Delta^\II \mid \exists (x,y) \in R^\II \ \mbox{such that} \ y \in C^\II\}$, for instance the extension of $\exists \mathit{tutoredBy}.\mathit{Professor}$ contains all the elements of the domain having \emph{at least} a professor as a tutor, i.e. those being related by role $\mathit{tutoredBy}$ to elements of the extension of $\mathit{Professor}$;
    \item $(\forall R.C)^\II = \{x \in \Delta^\II \mid \forall (x,y) \in R^\II \ \mbox{we have} \ y \in C^\II\}$, for instance the extension of $\forall \mathit{hasChild}.\mathit{Female}$ contains  parents whose children are \emph{all} female, in other words $x$ belongs to $(\forall \mathit{hasChild}.\mathit{Female})^\II$ if, for all pairs $(x,y)$ in $\mathit{hasChild}^\II$ we have that $y$ is a female (she belongs to the extension $\mathit{Female}^\II$).
\end{itemize}

One can reason about a DLs knowledge base by exploiting inference services like: \begin{itemize}
  \item satisfiability of the knowledge base: does the knowledge base admit a model?
  \item concept satisfiability: given a concept $C$, is there a model of the knowledge base assigning a non empty extension $C^\II$ to $C$?
  \item subsumption: given two concepts $C$ and $D$, is $C$ more general than $D$ (i.e. $D^\II \subseteq C^\II$) in any model of the knowledge base?
  \item instance checking: given an individual $a$ and a concept $C$, is $a$ an instance of $C$ (i.e. the domain element corresponding to $a$ belongs to $C^\II$) in any model of the knowledge base?
\end{itemize}

\noindent In the above example, one can infer, for instance, that the knowledge base is consistent, whereas it would be not if we added the inclusion $(*) \ \mathit{Student} \sqsubseteq \lnot \mathit{Worker}$ (students are not workers, therefore there does not exist any model for the knowledge base, since Max is a working student). The concept $\mathit{Worker} \sqcap \mathit{Student}$ (working student) is satisfiable, whereas it would be not in presence of the inclusion $(*)$.
Moreover, we can infer that Max is a tax payer $$\mathit{TaxPayer}(\mathit{max}),$$ and that all working students have an ID $$\mathit{Student} \sqcap \mathit{Worker} \sqsubseteq \exists \mathit{hasId}.(\mathit{Code} \sqcup \mathit{NationalInsuranceNumber}).$$

\subsection{Reasoning about typicality in DLs: the nonmonotonic logic $\alct$}
Since the very objective of the
TBox is to build a taxonomy of concepts, the need of representing
prototypical properties and of reasoning about defeasible
inheritance of such properties easily arises. The traditional
approach is to handle defeasible inheritance by integrating some
kind of nonmonotonic reasoning mechanism: this has led to study
nonmonotonic extensions of DLs [\cite{baader95a,donini,bonattilutz2}] with the objective of fulfilling the following desiderata for such an extension:

\begin{itemize}
\item The (nonmonotonic) extension must have a clear semantics and should be based on the same semantics as the underlying monotonic DL.
\item The extension should allow to specify prototypical properties in a natural and direct way.
\item The extension must be decidable, if  so is the underlying monotonic DL and, possibly,
 computationally effective.
\end{itemize}

\noindent A  simple but powerful nonmonotonic extension of DLs is proposed in
[\cite{AIJ,FI2009,dl2013,AIJ2014}]. In
this approach, ``typical'' or ``normal'' properties can be
directly specified  by means of a ``typicality'' operator $\tip$
enriching  the underlying DL.
The logic  so obtained is called $\alct$. The intuitive idea is that
$\tip(C)$ selects the {\em typical} instances of a concept $C$. We can therefore distinguish between the properties that
hold for all instances of concept $C$ ($C \sqsubseteq D$), and
those that only hold for the normal or typical instances of $C$
($\tip(C) \sqsubseteq D$).

The semantics of the $\tip$ operator can be given by means of a set of postulates that are  a reformulation of axioms and rules of nonmonotonic entailment in rational logic {\bf R} \ [\cite{whatdoes}]: in this respect an assertion of the form $\tip(C) \sqsubseteq D$ is equivalent to the conditional assertion $C \mathrel{{\scriptstyle\mid\!\sim}} D$ in {\bf R}. The basic ideas are as follows: given a domain $\Delta^\II$  and an evaluation function $.^\II$, one can define a function $f_\tip : Pow(\Delta^\II) \longmapsto Pow(\Delta^\II)$
that selects  the {\em typical} instances of any $S \subseteq \Delta^\II$; in
case $S = C^\II$ for a concept $C$, the selection function selects the typical instances
of $C$, namely: $$(\tip(C))^\II = f_\tip(C^\II).$$
$f_\tip$ has the following properties for all subsets $S$ of
$\Delta^\II$, that are essentially a restatement of the properties characterizing rational logic {\bf R}:

\vspace{0.2cm}
\noindent $(f_\tip-1) \ f_\tip(S) \subseteq S$ \\
\noindent $(f_\tip-2) \ \mbox{if} \ S \not=\emptyset \mbox{, then also} \ f_\tip(S) \not=\emptyset$ \\
\noindent  $(f_\tip-3) \ \mbox{if} \ f_\tip(S) \subseteq R, \mbox{then} \ f_\tip(S)=f_\tip(S \cap R)$ \\
\noindent  $(f_\tip-4) \ f_\tip(\bigcup S_i) \subseteq \bigcup f_\tip (S_i)$\\
\noindent  $(f_\tip-5) \ \bigcap f_\tip(S_i) \subseteq f_\tip(\bigcup S_i)$\\ 
\noindent  $(f_\tip-6) \ \mbox{if} \ f_\tip(S) \cap R \not=\emptyset, \mbox{then} \ f_\tip(S \cap R) \subseteq f_\tip(S)$
\vspace{0.2cm}

The semantics of the $\tip$ operator can be equivalently formulated in terms of
\emph{rational models} [\cite{AIJ2014}]: a model $\emme$ is any
structure $\langle \Delta^\II, <, .^\II \rangle$ where $\Delta^\II$ is the domain, 
 $<$ is an irreflexive, transitive, well-founded and modular (for all $x, y, z$ in $\Delta^\II$, if
$x < y$ then either $x < z$ or $z < y$) relation over
$\Delta^\II$. In this respect, $x < y$ means that $x$ is ``more normal'' than $y$, and that the typical members of a concept $C$ are the minimal elements of $C$ with respect to this relation. An element $x
\in \Delta^\II$ is a {\em typical instance} of some concept $C$ if $x \in
C^\II$ and there is no $C$-element in $\Delta^\II$ {\em more typical} than
$x$. Elements in $\Delta^\II$ are then organized in different ``levels'' or ``ranks'' by the modularity of $<$, where each element $x$ at rank $i$ is incomparable with each other (i.e., for each $y$ at rank $i$ neither $x<y$ nor $y<x$) and it is more normal than all elements with an higher rank $j>i$, therefore minimal elements of $C$ are those having the minimum rank.  In detail,  $.^\II$ is the extension function that maps each
concept $C$ to $C^\II \subseteq \Delta^\II$, and each role $R$
to  $R^\II \subseteq \Delta^\II \times \Delta^\II$. For concepts of
$\alc$, $C^\II$ is defined as usual.  For the $\tip$ operator, we have
$$(\tip(C))^\II = Min_<(C^\II),$$ where $Min_<(C^\II)=\{x \in C^\II \mid \not\exists y \in C^\II \ \mbox{s.t.}  \ y<x\}$.



 Given standard definitions of satisfiability of a KB in a model, we define a notion of entailment in $\alct$. Given a query $F$ (either an inclusion $C \sqsubseteq D$ or an assertion $C(a)$ or an assertion of the form $R(a,b)$), we say that $F$ is entailed from a KB if $F$ holds in all $\alct$ models satisfying KB.

Even if
the typicality operator $\tip$ itself  is nonmonotonic (i.e.
$\tip(C) \sqsubseteq E$ does not imply $\tip(C \sqcap D)
\sqsubseteq E$), what is inferred
from a KB can still be inferred from any KB' with KB $\subseteq$
KB', i.e. the logic $\alct$ is monotonic. In order to perform useful nonmonotonic inferences,   in
[\cite{AIJ2014}] the authors have strengthened  the above semantics by
restricting entailment to a class of minimal 
models. Intuitively, the idea is to
restrict entailment to models that \emph{minimize the atypical instances of a concept}. The resulting logic  corresponds to a notion of \emph{rational closure} on top of $\alct$. Such a notion is a natural extension of the rational closure construction provided  in [\cite{whatdoes}] for the propositional logic.

The nonmonotonic semantics of $\alct$ relies on minimal rational models  that
minimize the \emph{rank  of domain elements}. Informally, given two models of
KB, one in which a given domain element $x$ has rank 2 (because for instance
$z < y < x)$, and another in which it has rank 1 (because only
$y < x$), we prefer the latter,
as in this model the element $x$ is assumed to be ``more typical'' than in the former.


\begin{definition}[Rank of a domain element $k_{\emme}(x)$]\label{definition_rank}
Given a model $\emme=$$\langle \Delta^\II, <, .^\II\rangle$, the rank $k_{\emme}$  of a domain element $x \in \Delta^\II$, is the
length of the longest chain $x_0 < \dots < x$ from $x$
to a minimal $x_0$ (i.e. such that there is no ${x'}$ such that  ${x'} < x_0$).
\end{definition}

\begin{definition}[Minimal models]\label{Preference between models in case of fixed valuation} 
Given $\emme_1 = \langle \Delta^{\II_1}, <_1, .^{\II_1} \rangle$ and $\emme_2 =
\langle \Delta^{\II_2}, <_2, .^{\II_2} \rangle$ we say that $\emme_1$ is preferred to
$\emme_2$ if:
\begin{itemize}
\item $\Delta^{\II_1} = \Delta^{\II_2}$
\item $C^{\II_1} = C^{\II_2}$ for all concepts $C$
\item for all $x \in \Delta^{\II_1}$, it holds that $ k_{\emme_1}(x) \leq k_{\emme_2}(x)$ whereas
there exists $y \in \Delta^{\II_1}$ such that $ k_{\emme_1}(y) < k_{\emme_2}(y)$.
\end{itemize}
Given a knowledge base, we say that
$\emme$ is a minimal model  if it is a model satisfying the knowledge base and  there is no
$\emme'$ model satisfying it such that $\emme'$ is preferred to $\emme$.
\end{definition}

 Query entailment is then restricted to minimal {\em canonical models}. The intuition is that a canonical model contains all the individuals that enjoy properties that are consistent with KB. A model $\emme$ is a minimal canonical model of KB
if it satisfies KB, it is  minimal  and it is canonical\footnote{In Theorem 10 in [\cite{AIJ2014}] the authors have shown that for any consistent KB there exists a finite minimal canonical model of KB.}. A query $F$ is minimally entailed from a KB if it holds in all minimal canonical models of KB.

\begin{example}
   Let us consider and extend the example in the Introduction about athletes and sumo wrestlers. In the logic $\alct$ we can have a knowledge base whose TBox contains the following inclusions:

   \begin{quote}
 $\mathit{SumoWrestler} \sqsubseteq \mathit{Athlete}$ \\
 $\mathit{Athlete} \sqsubseteq \mathit{HumanBeing}$\\
 $\tip (\mathit{Athlete}) \sqsubseteq \mathit{Fit}$ \\
 $\tip (\mathit{SumoWrestler}) \sqsubseteq  \lnot \mathit{Fit}$\\
 $\tip (\mathit{Athlete}) \sqsubseteq \mathit{YoungPerson}$
\end{quote}

\noindent where standard inclusions are intended as usual in standard $\alc$: all sumo wrestlers are athletes, and all athletes are human beings. Typicality inclusions represent that, respectively, usually, athletes are  fit, typical sumo wrestlers are not  fit, and, normally, athletes are young persons.

The ABox contains the following facts about Roberto and Hiroyuki:
   \begin{quote}
 $\mathit{Athlete}(\mathit{roberto})$ \\
 $\mathit{SumoWrestler}(\mathit{hiroyuki})$
\end{quote}

\noindent used to represent that Roberto is an athlete, whereas Hiroyuky is a sumo wrestler.
Concerning its nonmonotonic reasoning capabilities, the logic $\alct$ allows one to infer from the above knowledge base that:
\begin{quote}
$\tip(Athlete) \sqsubseteq \lnot \mathit{SumoWrestler}$ \\
$\tip(\mathit{Athlete} \sqcap \mathit{Bald}) \sqsubseteq \mathit{Fit} \hfill (\mbox{\em irrelevance})$
\end{quote}
the last one stating that being bald is irrelevant with respect to being in fit. Furthermore, since we know that Hiroyuki is a sumo wrestler and Roberto is an athlete, we can infer  the following facts: 
\begin{quote}
$\mathit{Fit}(\mathit{roberto})$ \\
$\lnot \mathit{Fit}(\mathit{hiroyuky})\hfill (\mbox{\em specificity})$
\end{quote}
Observe that, in the last one, the logic gives preference to the most specific information (Hiroyuky is both an athlete and a sumo wrestler).
\end{example}

 In [\cite{AIJ2014}] it is shown that  query entailment in $\alct$  is in \textsc{ExpTime}.

\subsection{Probabilistic DLs: the DISPONTE semantics}
A probabilistic extension of Description Logics under the distribution semantics is proposed in [\cite{riguzzi}]. In this approach, called DISPONTE, the authors propose the integration of probabilistic information with DLs based on the distribution semantics for probabilistic logic programs [\cite{sato}]. The basic idea is to label inclusions of the TBox as well as facts of the ABox with a real number between 0 and 1, representing their probabilities (intended as \emph{degrees of belief}), assuming that each axiom is independent from each other. The resulting knowledge base defines a probability distribution over \emph{worlds}: roughly speaking, a world is obtained by choosing, for each axiom of the KB, whether it is considered as true of false. The distribution is further extended to queries and the probability of a query is obtained by marginalizing the joint distribution of the query and the worlds.

As an example, consider the following variant of the knowledge base inspired by the people and pets ontology  in [\cite{riguzzi}]:

\begin{quote}
   $ 0.3 \quad :: \quad \exists \mathit{hasAnimal}.\mathit{Pet} \sqsubseteq \mathit{NatureLover}$ \hfill(1)\\
   $ 0.6 \quad :: \quad \mathit{Cat} \sqsubseteq \mathit{Pet}$ \hfill(2)\\
   $ 0.9 \quad :: \quad \mathit{Cat}(\mathit{tom})$ \hfill(3)\\
   $ \mathit{hasAnimal}(\mathit{kevin},\mathit{tom})$ \hfill (4)
\end{quote}

\noindent The inclusion (1) expresses that individuals that own a pet are nature lovers with a 30\% probability, whereas  (2) is used to state that cats are pets with probability 60\%. The ABox fact (3) represents that Tom is a cat with probability 90\%. Inclusions (1), (2) and (3) are \emph{probabilistic} axioms, whereas (4) is a \emph{certain} axiom, that must always hold. The KB has the following eight possible worlds:
\begin{center}
    $\{ ((1),0), ((2),0), ((3),0) \}$ \\
    $\{ ((1),0), ((2),0), ((3),1) \}$ \\
    $\{ ((1),0), ((2),1), ((3),0) \}$ \\
    $\{ ((1),0), ((2),1), ((3),1) \}$ \\
    $\{ ((1),1), ((2),0), ((3),0) \}$ \\
    $\{ ((1),1), ((2),0), ((3),1) \}$ \\
    $\{ ((1),1), ((2),1), ((3),0) \}$ \\
    $\{ ((1),1), ((2),1), ((3),1) \}$
\end{center}
representing all possible combinations of considering/not considering each probabilistic axiom. For instance, the world $\{ ((1),1), ((2),0), ((3),1) \}$ represents the situation in which we have that (1) and (3) hold, i.e. $\exists \mathit{hasAnimal}.\mathit{Pet} \sqsubseteq \mathit{NatureLover}$ and $\mathit{Cat}(\mathit{tom})$, whereas (2) does not. The query $$\mathit{NatureLover}(\mathit{kevin})$$ is true only in the last world, i.e. having that (1), (2) and (3) are all true, whereas it is false in all the other ones. The probability of such a query is $P(\mathit{NatureLover}(\mathit{kevin}))=0.3 \times 0.6 \times 0.9=0.162$.

It is worth noticing that, in $\pfl$, we do not employ the whole characteristics of the DISPONTE semantics. In particular, as we will describe in detail in the next section, as far as the inferential capabilities are concerned, we exclusively adopt $\alct$, whereas we use DISPONTE only as a (necessary) ingredient to generate all different knowledge bases obtained by considering different subsets of typicality inclusions.

\section{ $\pfl$: A Logic for Commonsense Concept Combination}\label{sez:pfl}
In this section, we introduce a new nonmonotonic Description Logic $\pfl$ that combines the semantics based on the rational closure of $\alct$ [\cite{AIJ2014}] with the DISPONTE semantics [\cite{riguzzi,disponteijcai}] of probabilistic DLs.

By taking inspiration from [\cite{lieto2017dual,lieto2015common}], in our representational assumptions we consider two different types of properties associated to a given concept: rigid and typical. Rigid properties are those defining a concept, e.g. $C \sqsubseteq D$ (all $C$s are $D$s). Typical properties are represented by inclusions equipped by a degree of belief expressed through probabilities like in the DISPONTE Semantics. Additionally, as mentioned, we employ insights coming from the cognitive science for the determination of a dominance effect between the concepts to be combined, distinguishing between concept HEAD and MODIFIER. 
Since the conceptual combination is usually expressed via natural language we consider the following common situations: in a combination ADJECTIVE - NOUN (for instance, {\em red apple}) the HEAD is represented by the NOUN ({\em apple}) and the modifier by the ADJECTIVE ({\em red}). 
In the more complex case of NOUN-NOUN combinations (for instance, {\em pet fish}) usually the HEAD is represented by the last expressed concept ({\em fish} in this case). As we will see, however, in the NOUN-NOUN case (i.e. the one we will take into account in this paper) does not exists a clear rule to follow \footnote{It is worth-noting that a general framework for the automatic identification of a HEAD/MODIFIER combination is currently not available in literature. In this work we will take for granted that some methods for the correct identification of these pairs exist and we will focus on the reasoning part.}.

The language of $\pfl$ extends the basic DL $\mathcal{ALC}$ by \emph{typicality inclusions} of the form $\tip(C) \sqsubseteq D$ equipped by a real number $p \in (0.5,1]$ -- observe that the extreme $0.5$ is not included -- representing its degree of belief, whose  meaning is that ``we believe with degree/probability $p$ that, normally, $C$s are also $D$'' \footnote{The reason why we only allow typicality inclusions equipped with probabilities $p > 0.5$ is detailed in the following.}.

\begin{definition}[Language of $\pfl$]
We consider an alphabet of concept names $\mathtt{C}$, of role names
$\mathtt{R}$, and of individual constants $\mathtt{O}$.
Given $A \in \mathtt{C}$ and $R \in \mathtt{R}$, we define:

\vspace{0.2cm}

 $ C, D:= A \mid \top \mid \bot \mid  \lnot  C \mid  C \sqcap  C \mid  C \sqcup  C \mid \forall R. C \mid \exists R. C$
 
\vspace{0.2cm}

\noindent   We define a knowledge base $\kk=\langle \RR, \TT, \AAA \rangle$ where:

\noindent $\bullet$ $\RR$ is a finite set of rigid properties of the form $C \sqsubseteq D$;
    
\noindent $\bullet$ $\TT$ is a finite set of typicality properties of the form $$p \ :: \ \tip(C) \sqsubseteq D$$ where $p \in (0.5,1] \subseteq \mathbb{R}$ is the degree of belief of the typicality inclusion;
 
 \noindent $\bullet$ $\AAA$ is the ABox, i.e. a finite set of formulas of the form either $C(a)$ or $R(a,b)$, where $a, b \in \mathtt{O}$ and $R \in \mathtt{R}$.
\end{definition}

\begin{example}
   Let us consider and extend the previous example about athletes and sumo wrestlers. In the logic $\pfl$ we can have a knowledge base $\kk=\langle \RR, \TT, \AAA \rangle$ as follows:

\vspace{0.2cm}
\noindent $\RR$:
   \begin{itemize}
\item $\mathit{SumoWrestler} \sqsubseteq \mathit{Athlete}$ 
\item $\mathit{Athlete} \sqsubseteq \mathit{HumanBeing}$
\end{itemize}

\noindent $\TT$:
   \begin{itemize}
\item $0.8 \ :: \ \tip (\mathit{Athlete}) \sqsubseteq \mathit{Fit}$ 
\item $0.8 \ :: \ \tip (\mathit{SumoWrestler}) \sqsubseteq  \lnot \mathit{Fit}$
\item $0.95 \ :: \ \tip (\mathit{Athlete}) \sqsubseteq \mathit{YoungPerson}$
\end{itemize}

\noindent $\AAA$:
   \begin{itemize}
\item $\mathit{Athlete}(\mathit{roberto})$ 
\item $\mathit{SumoWrestler}(\mathit{hiroyuki})$
\end{itemize}

\noindent Rigid properties of $\RR$ are intended as usual in standard $\alc$: all sumo wrestlers are athletes, and all athletes are human beings.
Typicality properties of $\TT$ represent the following facts, respectively: 
\begin{itemize}
\item usually, athletes are  fit, and this fact has a degree of belief of $80\%$;
\item typical sumo wrestlers are not  fit with a degree of belief of $80\%$;
\item we believe with a degree of $95\%$ in the fact that, normally, athletes are young persons.
\end{itemize}

The ABox facts in $\AAA$ are used to represent that Roberto is an athlete, whereas Hiroyuky is a sumo wrestler.

We remind that, since we exploit the logic of typicality $\alct$, our logic $\pfl$ inherits its nonmonotonic reasoning capabilities. 
\end{example}

\noindent As mentioned above, we do not avoid typicality inclusions with degree 1. Indeed, an inclusion $1 \ :: \ \tip(C) \sqsubseteq D$ means that there is no uncertainty about a given typicality inclusion. On the other hand, since the very cognitive notion of typicality [\cite{rosch1975cognitive}] derives from that one of probability distribution \footnote{In the sense that prototypes are usually intended as statistically representative members of a category (e.g. the prototype of $\mathit{Bird}$ in, let's say, Europe is different with respect to the prototype of $\mathit{Bird}$ in New Zealand, i.e. a kiwi which do not fly, since the types of birds encountered by these two populations are statistically different and therefore the learned typical members differ as well). This assumption is also reflected in computational frameworks of cognitive semantics where prototypes for conceptual representations are naturally calculated/learned as centroids in vector space models of meaning[\cite{gardenfors2004conceptual, gardenfors2014geometry}].}, and this latter notion is also intrinsically connected to the one concerning the level of uncertainty/degree of belief associated to  typicality inclusions (i.e. typical knowledge is known to come with a low degree of uncertainty [\cite{lawry2009uncertainty}]), we only allow typicality inclusions equipped with degrees of belief $p > 0.5$\footnote{It is  worth noticing that  in our logic, the uncertain/graded component of typicality is captured by the \emph{ranked} semantics underlying the operator $\tip$ in the logic $\alct$. On the other hand, the epistemic uncertainty is modelled by the interpretation of probabilities of the DISPONTE semantics.}. Therefore,
in our effort of integrating two different semantics -- typicality based logic and DISPONTE  -- the choice of having degrees of belief/probabilities higher than $0.5$ for typicality inclusions seems to be the only one compliant with both the formalisms. In fact, despite the DISPONTE semantics [\cite{disponteijcai}] allows to assign also low degrees of belief to standard inclusions, in the logic $\pfl$, for what explained above, it would be at least counter-intuitive to also allow low degrees of belief for typicality inclusions (simply because typicality inclusions with high uncertainty do not describe any typical knowledge). For example, the logic $\pfl$ does not allow an inclusion like $0.3 \ :: \ \tip(\mathit{Student}) \sqsubseteq \mathit{YoungPerson}$, that could be interpreted as ``normally, students are not young people''. Please, note that this is not a limitation of the expressivity of the logic $\pfl$: we can in fact represent properties not holding for typical members of a category, for instance if one needs to represent that typical students are not married, we can have that $0.8 \ :: \ \tip(\mathit{Student}) \sqsubseteq  \lnot \mathit{Married}$, rather than  $0.2 \ :: \ \tip(\mathit{Student}) \sqsubseteq  \mathit{Married}$.

 Following the DISPONTE semantics, in $\pfl$  each typicality inclusion is independent from each other. This avoids the problem of dealing with degrees of inconsistent inclusions. Let us consider the following knowledge base:
\begin{quote}
   $\mathit{WorkingStudent} \sqsubseteq \mathit{Student}$\\
   $(i) \ 0.8 \ :: \ \mathit{Student} \sqsubseteq \lnot \mathit{WorkingTaxPayer}$ \\
   $(ii) \ 0.9 \ :: \ \mathit{WorkingStudent} \sqsubseteq \mathit{WorkingTaxPayer}$
\end{quote}

\noindent Also in the scenarios where both the conflicting typical inclusions $(i)$ and $(ii)$ are considered, the two degrees describe, respectively, that we believe that typical students do not pay working taxes with a degree of $80\%$, and that working students normally pay working taxes with a degree of $90 \%$, then those  inclusions are both acceptable due to the independence assumption. The two degrees will contribute to a definition of probability of such scenario (as we will describe in Definition \ref{def:scenario}). It is worth noticing that the underlying logic of typicality allows us to get for free the correct way of reasoning in this case, namely if the ABox contains the information that $\mathit{Mark}$ is a working student, we obtain that he pays working taxes, i.e. $\mathit{WorkingTaxPayer}(\mathit{Mark})$.

A model $\emme$ in the logic $\pfl$ extends standard $\mathcal{ALC}$ models by a preference relation among domain elements as in the logic of typicality [\cite{AIJ2014}]. In this respect, $x < y$ means that $x$ is ``more normal'' than $y$, and that the typical members of a concept $C$ are the minimal elements of $C$ with respect to this relation\footnote{It could be possible to consider an alternative semantics whose models are equipped with multiple preference relations, whence with multiple typicality operators. In this case, it should be possible to distinguish different aspects of exceptionality, however the approach based on a single preference relation in [\cite{AIJ2014}] ensures good computational properties (reasoning in the resulting nonmonotonic logic $\alct$ has the same complexity of the standard $\alc$), whereas adopting multiple preference relations could lead to higher complexities.}. An element $x
\in \Delta^\II$ is a {\em typical instance} of some concept $C$ if $x \in
C^\II$ and there is no $C$-element in $\Delta^\II$ {\em more normal} than
$x$. Formally:

\begin{definition}[Model of $\pfl$]
A model $\emme$ is any
structure $$\langle \Delta^\II, <, .^\II \rangle$$ where: 
\begin{itemize}
\item $\Delta^\II$ is a non empty set of items called the domain;
\item $<$ is an irreflexive, transitive, well-founded and modular (for all $x, y, z$ in $\Delta^\II$, if
$x < y$ then either $x < z$ or $z < y$) relation over
$\Delta^\II$;
\item $.^\II$ is the extension function that maps each atomic
concept $C$ to $C^\II \subseteq \Delta^\II$, and each role $R$
to  $R^\II \subseteq \Delta^\II \times \Delta^\II$, and is extended to complex concepts as follows:

\begin{itemize}
   \item $(\lnot C)^\II = \Delta^\II \ \backslash \ C^\II$
   \item $(C \sqcap D)^\II = C^\II \cap D^\II$
   \item $(C \sqcup D)^\II = C^\II \cup D^\II$
   \item $(\exists R.C)^\II = \{x \in \Delta^\II \mid \exists (x,y) \in R^\II \ \mbox{such that} \ y \in C^\II \}$
   \item $(\forall R.C)^\II = \{x \in \Delta^\II \mid \forall (x,y) \in R^\II \ \mbox{we have} \ y \in C^\II \}$
   \item $(\tip(C))^\II = Min_<(C^\II),$ where $Min_<(C^\II)=\{x \in C^\II \mid \not\exists y \in C^\II \ \mbox{s.t.}  \ y<x\}$.
\end{itemize}
\end{itemize}
\end{definition}

\noindent A model $\emme$  can be equivalently defined by postulating the existence of
a function $k_{\emme}: \Delta^\II \longmapsto \mathbb{N}$, where $k_{\emme}$ assigns a finite rank to each domain element [\cite{AIJ2014}]: the rank of $x$ is the
length of the longest chain $x_0 < \dots < x$ from $x$
to a minimal $x_0$, i.e. such that there is no ${x'}$ such that  ${x'} < x_0$. The rank function $k_{\emme}$ and $<$ can be defined from each other  by letting $x < y$ if and only if $k_{\emme}(x) < k_{\emme}(y)$.


\begin{definition}[Model satisfying a knowledge base in $\pfl$]\label{semantica}
Let $\kk=\langle \RR, \TT, \AAA \rangle$
be a KB. Given a model $\emme=\langle \Delta^{\mathcal{I}}, <, .^{\mathcal{I}}\rangle$, we assume that $.^\II$ is extended  to assign a  domain element
$a^{\mathcal{I}}$ of $\Delta^\II$ to each individual constant $a$ of $\mathtt{O}$.
We say that:
\begin{itemize}
\item $\emme$  satisfies $\RR$ if, for all $C \sqsubseteq D \in \RR$, we have $C^\II \subseteq D^\II$;
\item $\emme$ satisfies $\TT$ if, for all $q \ :: \ \tip(C) \sqsubseteq D \in \TT$, we have that\footnote{It is worth noticing that here the degree $q$ does not play any role. Indeed, a typicality inclusion $\tip(C) \sqsubseteq D$ holds in a model only if it satisfies the semantic condition of the underlying DL of typicality, i.e. minimal (typical) elements of $C$ are elements of $D$. The degree of belief $q$ will have a crucial role in the application of the distributed semantics, allowing the definition of scenarios as well as the computation of their probabilities.} $\tip(C)^\II \subseteq D^\II$, i.e. $Min_<(C^\II) \subseteq D^\II$; 
\item $\emme$ satisfies $\AAA$ if, for all assertion $F \in \AAA$, if $F = C(a)$ then $a^{\mathcal{I}} \in C^{\mathcal{I}}$, otherwise
          if $F = R(a,b)$ then $(a^{\mathcal{I}},b^{\mathcal{I}}) \in R^{\mathcal{I}}$.
\end{itemize}
\end{definition}

Even if the typicality operator $\tip$ itself  is nonmonotonic (i.e.
$\tip(C) \sqsubseteq E$ does not imply $\tip(C \sqcap D)
\sqsubseteq E$), what is inferred
from a KB can still be inferred from any KB' with KB $\subseteq$
KB', i.e. the resulting logic is monotonic. As already mentioned in Section \ref{sez:background}, in order to perform useful nonmonotonic inferences,   in
[\cite{AIJ2014}] the authors have strengthened  the above semantics by
restricting entailment to a class of minimal 
models. Intuitively, the idea is to
restrict entailment to models that \emph{minimize the atypical instances of a concept}. The resulting logic corresponds to a notion of \emph{rational closure} on top of $\alct$. Such a notion is a natural extension of the rational closure construction provided  in [\cite{whatdoes}] for the propositional logic.  This nonmonotonic semantics relies on minimal rational models  that minimize the \emph{rank  of domain elements}. Informally, given two models of KB, one in which a given domain element $x$ has rank 2 (because for instance
$z < y < x)$, and another in which it has rank 1 (because only
$y < x$), we prefer the latter,
as in this model the element $x$ is assumed to be ``more typical'' than in the former.
Query entailment is then restricted to minimal {\em canonical models}. The intuition is that a canonical model contains all the individuals that enjoy properties that are consistent with KB. This is needed when reasoning about the rank of the concepts: it is important to have them all represented. 
A query $F$ is minimally entailed from a KB if it holds in all minimal canonical models of KB.
 In [\cite{AIJ2014}] it is shown that  query entailment in the nonmonotonic $\alct$ 
 is in \textsc{ExpTime}.

\begin{definition}[Entailment]\label{entailment in pfl}
Let $\kk=\langle \RR, \TT, \AAA \rangle$
be a KB and let $F$ be either $C \sqsubseteq D$ ($C$ could be $\tip(C')$) or $C(a)$ or $R(a,b)$. We say that $F$ follows from $\kk$ if, for all minimal $\emme$ satisfying $\kk$, then $\emme$ also satisfies $F$. 
\end{definition}

\noindent Let us now define the notion of \emph{scenario} of the composition of concepts. Intuitively, a scenario is a knowledge base obtained by adding to all rigid properties in $\RR$ and to all ABox facts in $\AAA$ only \emph{some} typicality properties. More in detail, we define an {\em atomic choice} on each typicality inclusion, then we define a {\em selection} as a set of atomic choices in order to select which typicality inclusions have to be considered in a scenario.

\begin{definition}[Atomic choice]
Given $\kk=\langle \RR, \TT, \AAA \rangle$, where $\TT = \{ E_1 = q_1 \ :: \tip(C_1) \sqsubseteq D_1,  \dots, E_n = q_n \ :: \tip(C_n) \sqsubseteq D_n \}$ we define ($E_i$, $k_i$) an \emph{atomic choice}, where $k_i \in \{0, 1\}$. 
\end{definition}

\begin{definition}[Selection]\label{def:selection}
Given  $\kk=\langle \RR, \TT, \AAA \rangle$, where $\TT = \{ E_1 = q_1 \ :: \tip(C_1) \sqsubseteq D_1,  \dots, E_n = q_n \ :: \tip(C_n) \sqsubseteq D_n \}$ and a set of atomic choices $\nu$, we say that $\nu$ is a \emph{selection} if, for each $E_i$, one decision is taken, i.e. either ($E_i$, 0) $\in \nu$ and ($E_i$, 1) $\not\in \nu$ or  ($E_i$, 1) $\in \nu$ and ($E_i$, 0) $\not\in \nu$ for $i=1, 2, \dots, n$. The probability of  $\nu$ is $P(\nu) = \prod\limits_{(E_i,1) \in \nu} q_i \prod\limits_{(E_i,0) \in \nu} (1-q_i)$.
\end{definition}

\begin{definition}[Scenario]\label{def:scenario}
Given  $\kk=\langle \RR, \TT, \AAA \rangle$, where $\TT = \{ E_1 = q_1 \ :: \tip(C_1) \sqsubseteq D_1, \dots, E_n = q_n \ :: \tip(C_n) \sqsubseteq D_n \}$ and given a selection $\sigma$, we define a \emph{scenario} $w_\sigma=\langle \RR,  \{E_i \mid (E_i, 1) \in \sigma\}, \AAA \rangle.$  
We also define the probability of a scenario $w_\sigma$ as the probability of the corresponding selection, i.e. $P(w_\sigma)=P(\sigma)$.
Last, we say that a scenario is \emph{consistent} with respect to $\kk$ when it admits a model in the logic $\pfl$ satisfying $\kk$.
\end{definition} 

\noindent We denote with $\allworlds$ the set of all scenarios, essentially the set of knowledge bases obtained by considering all possible subsets of typicality inclusions in $\TT$. It immediately follows that the probability of a scenario $P(w_\sigma)$ is a probability distribution over scenarios, that is to say $\sum\limits_{w \in \allworlds} P(w) = 1$.

\noindent Given a KB $\kk=\langle \RR, \TT, \AAA \rangle$ and given two concepts $C_H$ and $C_M$ occurring in $\kk$, our logic allows defining the compound concept $C$ as the combination of the HEAD $C_H$ and the MODIFIER $C_M$, where  the typical properties of the form $\tip(C) \sqsubseteq D$ (or, equivalently, $\tip(C_H \sqcap C_M) \sqsubseteq D$) to ascribe to the concept $C$ are obtained in the set of scenarios that: 

\begin{enumerate}
   \item are consistent with respect to $\kk$ in presence of at least a $C$-element, in other words the knowledge base 
   extending $\kk$ with the properties ascribed to the combined concept $C$ in the scenario, i.e. $\langle \RR, \TT \cup \ \{\tip(C_H \sqcap C_M) \sqsubseteq D \mid \mbox{either} \ \tip(C_H) \sqsubseteq D \in w_\sigma \ \mbox{or} \ \tip(C_M) \sqsubseteq D \in w_\sigma\}, \AAA \ \cup \ \{(C_H \sqcap C_M)(x)\} \rangle$, where $x$ does not occur in $\AAA$, admits a model in $\pfl$; 
   \item are not trivial, i.e. the scenarios with the highest probability considering either \emph{all} properties that can be consistently ascribed to $C$ are discarded or \emph{all} properties of the HEAD that can be consistently ascribed to $C$ are discarded \footnote{This choice is motivated by the challenges provided by the task of commonsense conceptual combination itself: in order to generate plausible novel compounds it is necessary to maintain a certain level of ``surprise'' in the combination, since obvious inheritance of attributes does not have any explanatory power for human-like and human-level concept combination [\cite{hampton1987inheritance}]. For this reason, both scenarios inheriting all the properties of the two concepts and all the properties of the HEAD are discarded. In this respect, the typicality-based inheritance procedure of our logic falls within the so-called \emph{functional compositionality}, as introduced in [\cite{pelletier}].};
   \item are those giving preference to the typical properties of the HEAD $C_H$ (with respect to those of the MODIFIER $C_M$) with the highest probability, that is to say a scenario $w$ is discarded if, in case of conflicting properties $D$ and $\lnot D$, $w$ contains an inclusion $p_1 \ :: \ \tip(C_M) \sqsubseteq \lnot D$ whereas it does not include another inclusion $p_2 \ :: \ \tip(C_H) \sqsubseteq  D$.
\end{enumerate}

\noindent In order to select the resulting scenarios we apply points 1, 2, and 3 above to blocks of scenarios with the same probability, in decreasing order starting from the highest one. More in detail, we first discard all the inconsistent scenarios, then we consider the remaining (consistent) ones in decreasing order by their probabilities. We then consider the blocks of scenarios with the same probability, and we proceed as follows:
\begin{itemize}
   \item we discard those considered as \emph{trivial}, consistently inheriting  all the properties from the HEAD (therefore, also scenarios inheriting all the properties of HEAD and MODIFIER are discarded) from the starting concepts to be combined;
   \item among the remaining ones, we discard those inheriting properties from the MODIFIER in conflict with properties that could be consistently inherited from the HEAD;
   \item if the set of scenarios of the current block is empty, i.e. all the scenarios have been discarded either because trivial or because preferring the MODIFIER, we repeat the procedure by considering the block of scenarios, all having the immediately lower probability;
   \item the set of remaining scenarios are those selected by the logic $\pfl$.
\end{itemize}

 More formally, our mechanism is described in Algorithm \ref{algoritmo}. Please note that this block-based procedure extends a previously developed method that simply selected the consistent scenarios with the probability range immediately lower to the non-trivial ones [\cite{ismis2018}]. Notice also that, in the initial knowledge base $\kk$, we have that the set of typicality inclusions is $\TT' \ \cup \ \{q_1 \ :: \ \tip(C_H) \sqsubseteq D_1, \dots, q_k \ :: \ \tip(C_H) \sqsubseteq D_k, q_{k+1} \ :: \ \tip(C_M) \sqsubseteq D_{k+1}, q_n \ :: \ \tip(C_M) \sqsubseteq D_n\}$, where $\TT'$ does not contain neither inclusions of the HEAD of the form $q_i \ :: \ \tip(C_H) \sqsubseteq D_i$ nor inclusions of the MODIFIER of the form $q_i \ :: \ \tip(C_M) \sqsubseteq D_i$ (Algorithm \ref{algoritmo}, line 2).
 
\begin{algorithm*}[!ht]
\caption{Algorithm for concept combination}\label{algoritmo}
\begin{algorithmic}[1]
\Procedure{ConceptCombination}{$\kk = \langle \RR, \TT, \AAA \rangle, C_H, C_M$}
\State let $\TT$ be $\TT' \ \cup \ \{q_1 \ :: \ \tip(C_H) \sqsubseteq D_1, \dots, q_k \ :: \ \tip(C_H) \sqsubseteq D_k, q_{k+1} \ :: \ \tip(C_M) \sqsubseteq D_{k+1}, q_n \ :: \ \tip(C_M) \sqsubseteq D_n\}$
\State Sel $\gets$ compute all the $2^n$ possible selections 
\State Snr $\gets \emptyset$ \Comment{Compute all $2^n$ scenarios}
\For {each $\nu \in$ Sel}    
   \State $P(\nu) \gets 1$    \Comment{Compute the probability of the current selection/scenario}
   \For {$i=1, \dots, n$}
      \If {$\nu_i = 1$}   
      	\State $P(\nu) \gets P(\nu) \times q_i$
        \Else
          \State $P(\nu) \gets P(\nu) \times (1-q_i)$
       \EndIf
   \EndFor
   \State $\TT_\nu \gets \emptyset$  \Comment{Build the scenario corresponding to $\nu$}
   \For {$i=1, \dots, n$} 
      \If {$\nu_i = 1$}
      	\State $\TT_\nu \gets \TT_\nu \cup \{q_i \ :: \ \tip(C_H \sqcap C_M) \sqsubseteq D_i\}$   
      \EndIf
   \EndFor
   \State Snr $\gets$ Snr $\cup \ \{ \langle \RR, \TT' \cup \TT_\nu, \AAA \rangle \}$
\EndFor
\State ConsSnr $\gets \emptyset$ \Comment{Discard inconsistent scenarios (reasoning in $\alct$)}
\For {each $\langle \RR, \TT \cup \TT_\nu, \AAA \rangle \in$ Snr} 
   \If {$\langle \RR, \TT \cup \TT_\nu, \AAA \ \cup \{(C_H \sqcap C_M)(x)\}\rangle$ is consistent in $\alct$}  
   \Comment{$x$ does not occur in $\AAA$}
       \State ConsSnr $\gets$ ConsSnr $ \ \cup \ \{\langle \RR, \TT \cup \TT_\nu, \AAA \rangle\}$
   \EndIf
\EndFor
\State Ord $\gets$ order scenarios in ConsSnr by probabilities $P(\nu)$ in a decreasing order
\State ResultingSnr $\gets \ \emptyset$   \Comment{The set ResultingSnr will contain the selected scenarios}
\While {ResultingSnr $== \emptyset$}   \Comment{Continue with the next block of scenarios}
   \State $w$ $\gets$ first scenario in Ord
   \State Max $\gets$ $P(w)$ \Comment{Highest probability in Ord}
   \State CurrentBlock $\gets \{w\}$  \Comment{Build the current set scenarios with the highest probability}
   \While{$P(w) ==$ Max}
      \State $w$ $\gets$ next scenario in Ord  \Comment{$w$ is not removed from Ord} 
      \If{$P(w) ==$ Max} 
         \State CurrentBlock $\gets$ CurrentBlock $ \ \cup \ \{w\}$
         \State remove $w$ from Ord
      \EndIf
   \EndWhile
   \For {each $w \in$ CurrentBlock}
        \If {$w$ not contains  all properties of $C_H$} \Comment{Trivial scenario to be discarded}
           \If {\textsc{ConflictHeadModifier}($w$,$\kk$,$C_H$,$C_M$) $==\mathit{false}$} 	\State \Comment{Scenario preferring MODIFIER to HEAD (see Algorithm \ref{algoritmoconflitti}). Scenario to be discarded}
		\State ResultingSnr $\gets$ ResultingSnr $\ \cup \ \{ w \}$ \Comment{Selected scenario}
           \EndIf
        \EndIf
   \EndFor
\EndWhile
\State \Return ResultingSnr
\EndProcedure
\end{algorithmic}
\end{algorithm*}

\begin{algorithm*}[!ht]
\caption{Algorithm  checking for a scenario giving preference to the HEAD}\label{algoritmoconflitti}
\begin{algorithmic}[1]
\Procedure{ConflictHeadModifier}{$w$,$\kk$,$C_H$,$C_M$}
   \For {each $q_j \ :: \ \tip(C_M) \sqsubseteq D \in w$}
      \For {each $q_k \ :: \ \tip(C_H) \sqsubseteq E \in \kk$ such that $q_k \ :: \ \tip(C_H) \sqsubseteq E \not\in w$}
         \If {$\kk \ \cup \ \{D(x), E(x)\}$ is inconsistent in $\alct$, $x \not\in \kk$}  
         \State \Comment{Properties $D$ and $E$ are contradictory in $\alct$}
            \State \Return $\mathit{true}$
         \EndIf
       \EndFor
   \EndFor
\State \Return $\mathit{false}$
\EndProcedure
\end{algorithmic}
\end{algorithm*}

 Lastly, we define the ultimate output of our mechanism: a knowledge base in the logic $\pfl$ whose set of typicality properties is enriched by those of the compound concept $C$.
Given a scenario $w$ satisfying the above properties, we define the properties of  $C$ as the set of inclusions $p \ :: \ \tip(C) \sqsubseteq D$, for all $\tip(C) \sqsubseteq D$ that are entailed (Definition \ref{entailment in pfl}) from $w$ in the logic $\pfl$. The probability $p$ is such that:
\begin{itemize}
 \item if $\tip(C_H) \sqsubseteq D$ is entailed from $w$, that is to say $D$ is a property inherited either from the HEAD (or from both the HEAD and the MODIFIER),  then $p$ corresponds to the probability of such inclusion of the HEAD in the initial knowledge base, i.e. $p \:: \ \tip(C_H) \sqsubseteq D \in \TT$;
 \item otherwise, i.e. $\tip(C_M) \sqsubseteq D$ is entailed from $w$, then $p$ corresponds to the probability of such inclusion of a MODIFIER in the initial knowledge base, i.e. $p \:: \ \tip(C_M) \sqsubseteq D \in \TT$.
\end{itemize}

The knowledge base obtained as the result of combining concepts $C_H$ and $C_M$ into the compound concept $C$ is called $C$-\emph{revised} knowledge base, and it is defined as follows: 
 $$\kk_C=\langle \RR, \TT \cup \{p \:: \ \tip(C) \sqsubseteq D\}, \AAA \rangle,$$  for all $D$ such that either $\tip(C_H) \sqsubseteq D$ is entailed in $w$ or $\tip(C_M) \sqsubseteq D$ is entailed in $w$ by Definition \ref{entailment in pfl}, and $p$ is defined as above.

Let us now define the probability that a query is entailed from a $C$-revised knowledge base. We restrict our concern to ABox facts. The intuitive idea is that, given a query $F$ of the form $A(a)$ and its associated probability $p$, the probability of $F$ is the product of $p$ and the probability of the inclusion in the $C$-revised knowledge base which is responsible for that. 

\begin{definition}[Probability of query entailment]\label{probability query}
Given a knowledge base $\kk= \langle \RR, \TT, \AAA \rangle$, the $C$-revised knowledge base $\kk_C$, a query $A(a)$ and its probability $p \in (0,1]$, we define the probability of the entailment of the query $A(a)$ from $\kk_C$, denoted as $\PP(A(a),p)$ as follows:

\begin{itemize}
\item  $\PP(A(a),p)=0$, if $A(a)$ is \emph{not} entailed from $\kk_C$;
 
\item $\PP(A(a),p)=p \times q$, where either $q \ :: \ \tip(C) \sqsubseteq A$ belongs to $\kk_C$ or $q \ :: \ \tip(C) \sqsubseteq D$ belongs to $\kk_C$ and $D \sqsubseteq A$ is entailed from $\RR$ in standard $\mathcal{ALC}$, otherwise.
\end{itemize}
\end{definition}


\noindent We conclude this section by showing that reasoning in $\pfl$ remains in the same complexity class of standard $\alc$ Description Logics. 
\begin{theorem}
Reasoning in $\pfl$ is \textsc{ExpTime}-complete.
\end{theorem}
\begin{proof}
For the completeness, let $n$ be the size of KB, then the number of typicality inclusions is $O(n)$. It is straightforward to observe that we have an exponential number of different scenarios, for each one we need to check whether the resulting KB is consistent in $\alct$ which is \textsc{ExpTime}-complete. Hardness immediately follows form the fact that $\pfl$ extends standard $\alc$. Reasoning in the revised knowledge base relies on reasoning in $\alct$, therefore we can conclude that
reasoning in $\pfl$ is \textsc{ExpTime}-complete.
\end{proof}

\section{Applications of the logic $\pfl$}\label{sez:esempi}
We propose three different types of examples adopting the logic $\pfl$, along with its embedded HEAD-MODIFIER heuristic, to model the phenomenon of typicality-based conceptual combination. In the first case ({\em pet fish}) we show how our logic is able to handle, under certain assumptions, this concept composition which is problematic for other formalisms. In the second case ({\em Linda the feminist bank teller}) we show how $\pfl$ is able to model the well known conjunction fallacy problem [\cite{tversky1983extensional}]. In the third case ({\em stone lion}) we show how our logic is also able to account for complex form of metaphorical concept combination. All these examples do not come \emph{ex-abrupto}, since they represent classical challenging cases to model in the field of cognitive science and cognitive semantics (see e.g. [\cite{lewis2016hierarchical}]) and have been showed in that past problematic to model by adopting other kinds of logics (for example fuzzy logic, [\cite{osherson1981adequacy,smith1984conceptual,hampton2011conceptual}]). 

In addition, we  exploit $\pfl$ to present an example of a possible application in the area of creative generation of new characters.
Finally, we show that the logic $\pfl$ can be iteratively applied to combine concepts already resulting from the combination of concepts. This type of iterative process has been never provided in previous formalizations trying to address similar or the very same phenomena, (e.g. in [\cite{lewis2016hierarchical,eppe2018computational}]). We show that the procedures provided in $\pfl$ are robust and consistent enough also for dealing with higher, iterative, levels of prototype-based compositionality.



\subsection{Pet Fish}\label{sezione pet fish}
In this section we exploit, in two different setups, the logic $\pfl$ in order to define the typical properties of the concept \emph{pet fish}, obtained as the combination of the concepts \emph{Pet} and \emph{Fish}. As mentioned before, this represents a well known and  paradigmatic example in cognitive science. The problem of combining the prototype of a pet with those of a fish is the following: a typical pet is affectionate and warm, whereas a pet fish is not; on the other hand, as a difference with a typical fish, a pet fish is not greyish, but it inherits its being scaly.

Let us consider, as a first setup, the following situation: let $\kk=\langle \RR, \TT, \AAA \rangle$ be a  KB, where the ABox $\AAA$ is empty, the set of rigid inclusions is
    $$\RR=\{\mathit{Fish} \sqsubseteq \forall \mathit{livesIn}.\mathit{Water}\}$$
and the set of typicality properties $\TT$ is as follows:
\begin{quote}
  \begin{enumerate}
      \item $0.8  \ :: \ \tip(\mathit{Fish}) \sqsubseteq \lnot \mathit{Affectionate}$ 
    \item $0.6 \ :: \ \tip(\mathit{Fish}) \sqsubseteq \mathit{Greyish}$ 
    \item $0.9 \ :: \ \tip(\mathit{Fish}) \sqsubseteq \mathit{Scaly}$ 
    \item $0.8 \ :: \ \tip(\mathit{Fish}) \sqsubseteq \lnot \mathit{Warm}$ 
\item $0.9 \ :: \ \tip(\mathit{Pet}) \sqsubseteq \forall \mathit{livesIn}.(\lnot \mathit{Water})$
   \item $0.8  \ :: \ \tip(\mathit{Pet}) \sqsubseteq \mathit{Affectionate}$ 
    \item $0.8 \ :: \ \tip(\mathit{Pet}) \sqsubseteq \mathit{Warm}$ 
  \end{enumerate}
\end{quote}
\vspace{-0.1cm}


\noindent By the properties of the typicality operator $\tip$, reasoning relying on the underlying $\alct$ we have that $$(*) \ \tip(\mathit{Pet} \sqcap \mathit{Fish}) \sqsubseteq \forall \mathit{livesIn}.\mathit{Water}.$$ Indeed, $\mathit{Fish} \sqsubseteq \forall \mathit{livesIn}.\mathit{Water}$ is a rigid property, which is always preferred to a typical one: in this case, additionally, the rigid property is also associated to the HEAD element fish.
 Therefore, this element is reinforced.
 
 Since $\mid \TT \mid = 7$, we have $2^7=128$ different scenarios.  
We can observe that some of them are not consistent\footnote{The inconsistency arises when the knowledge base is extended by two contradicting properties for the combined concept, for instance the knowledge base extended by both $\tip(\mathit{Pet} \sqcap \mathit{Fish}) \sqsubseteq \mathit{Warm}$ and $\tip(\mathit{Pet} \sqcap \mathit{Fish}) \sqsubseteq \lnot \mathit{Warm}$ would be consistent only if there are no pet fishes.}, more precisely those 
\begin{itemize}
\item[(i)] containing the inclusion $5$, thus contradicting $(*)$;
\item[(ii)] containing both inclusions $1$ and $6$;
\item[(iii)] containing both inclusions $4$ and $7$.    
\end{itemize}
It is worth noticing that this setting of the example represents one of worst cases in our analysis: indeed, the probabilities associated to the properties in $\TT$ related to the MODIFIER are not lower than the ones associated to the properties in $\TT$ related to the HEAD\footnote{It is worth noticing that the logic $\pfl$ would select the same scenario also in the more challenging situation in which degrees of properties of the HEAD are strictly lower, for instance in case inclusion $1$ would be replaced in $\TT$ by $0.7 \ :: \ \tip(\mathit{Fish}) \sqsubseteq \lnot \mathit{Affectionate}$.}. 

The scenario with the highest probability (up to $20 \%$) is both trivial and inconsistent: indeed, since probabilities $p_i$ equipping typicality inclusions are such that $p_i > 0.5$ by definition, we immediately have that the higher is the number of inclusions belonging to a scenario the higher is the associated probability. Since typicality inclusions introduce properties that are pairwise inconsistent, it follows that such scenarios must be discarded (Algorithm \ref{algoritmo}, from line 17).


As described in the previous section, we consider the other blocks of consistent scenarios considering their probabilities in descending order (Algorithm \ref{algoritmo}, lines from 21). Figure \ref{pfffigure} shows the $2^7 = 128$ different scenarios, one row for each scenario. 

\begin{figure}
\centering
\includegraphics[width=1.2\textwidth]{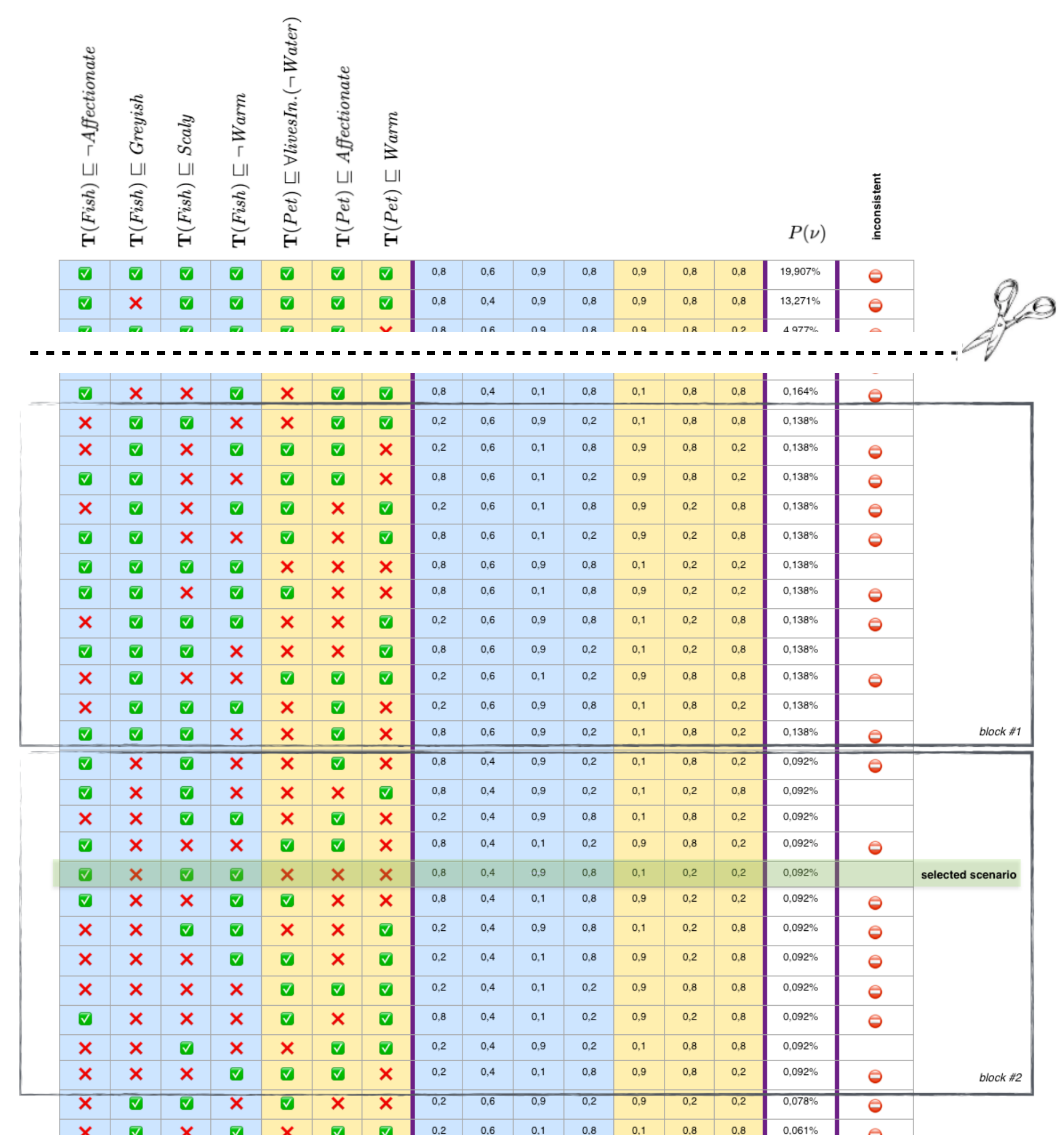}
\vspace{-0.3cm}
\caption{\label{pfffigure}Scenarios of the combination of $\mathit{Pet}$ and $\mathit{Fish}$. For the sake of readability we cut off the first part of the figure since the first blocks of the combination, i.e. 52 scenarios, are all inconsistent.}
\end{figure}


All scenarios with probabilities ranging from $19.907\%$ down to $0.164\%$ are inconsistent.
The first valid block contains scenarios whose probability is $0.138 \%$: the four consistent scenarios of this block, however, are discarded. Indeed, they either contain the inclusions $6$ but not $1$ or $7$ but not $4$, namely they give preference to the MODIFIER concerning a conflicting property of the HEAD, or they are trivial, i.e. they inherit all the properties of the HEAD (Algorithm \ref{algoritmo}, lines 32-35).

The next block contains four scenarios with probability of $0.092 \%$. The first two scenarios, again, either contain inclusions $6$ and  not  $1$ or contain  $7$ and  not  $4$, namely again it privileges the MODIFIER with respect to the corresponding negation in the HEAD. Therefore, these scenarios are discarded. The same for the last one, where both $6$ and $7$ are included rather than $1$ and $4$. 
The remaining scenario of this block includes three out of four properties of the HEAD, therefore it is not trivial and it is selected by the logic $\pfl$ for the composition of the two initial prototypes.



In conclusion, in our proposal, the not trivial scenario defining prototypical properties of a pet fish is defined from the selection $\sigma=\{(1,1), (2,0), (3,1), (4,1), (5,0), (6,0), (7,0)\}$,  and  contains inclusions 1, 3, and 4,
and the resulting scenario $w_\sigma$ is as follows:
\begin{quote}
  \begin{enumerate}
   \item[1.] $0.8  \ :: \ \tip(\mathit{Fish}) \sqsubseteq \lnot \mathit{Affectionate}$ 
   \item[3.] $0.9 \ :: \ \tip(\mathit{Fish}) \sqsubseteq \mathit{Scaly}$ 
   \item[4.] $0.8 \ :: \ \tip(\mathit{Fish}) \sqsubseteq \lnot \mathit{Warm}$ 
  \end{enumerate}
\end{quote}

\noindent The resulting $\mathit{Pet} \ \sqcap \ \mathit{Fish}$-revised knowledge base that the logic $\pfl$ suggests is as follows:

$$\kk_{\mathit{Pet} \ \sqcap \ \mathit{Fish}}=\langle \{\mathit{Fish} \sqsubseteq \forall \mathit{livesIn}.\mathit{Water}\}, \TT', \emptyset \rangle,$$ where $\TT$ is:
  \begin{itemize}
   \item $0.8  \ :: \ \tip(\mathit{Fish}) \sqsubseteq \lnot \mathit{Affectionate}$ 
    \item $0.6 \ :: \ \tip(\mathit{Fish}) \sqsubseteq \mathit{Greyish}$ 
    \item $0.9 \ :: \ \tip(\mathit{Fish}) \sqsubseteq \mathit{Scaly}$ 
    \item $0.8 \ :: \ \tip(\mathit{Fish}) \sqsubseteq \lnot \mathit{Warm}$ 
   \item $0.9 \ :: \ \tip(\mathit{Pet}) \sqsubseteq \forall \mathit{livesIn}.(\lnot \mathit{Water})$
      \item $0.8  \ :: \ \tip(\mathit{Pet}) \sqsubseteq \mathit{Affectionate}$ 
    \item $0.8 \ :: \ \tip(\mathit{Pet}) \sqsubseteq \mathit{Warm}$ 
    \item $0.8 \ :: \ \tip(\mathit{Pet} \sqcap \mathit{Fish}) \sqsubseteq \lnot \mathit{Affectionate}$ 
   \item $0.9 \ :: \ \tip(\mathit{Pet} \sqcap \mathit{Fish}) \sqsubseteq \mathit{Scaly}$
   \item $0.8 \ :: \ \tip(\mathit{Pet} \sqcap \mathit{Fish}) \sqsubseteq \lnot \mathit{Warm}$ 
\end{itemize}

\noindent Notice that in our logic $\pfl$, adding a new inclusion $\tip(\mathit{Pet} \sqcap \mathit{Fish}) \sqsubseteq \mathit{Red}$, would not be problematic.
(i.e. this means that our formalism is able to tackle the phenomenon of prototypical \emph{attributes emergence} for the new compound concept, a well established effect within the cognitive science literature [\cite{hampton1987inheritance}]).

Let us now consider the PET FISH problem by adopting a starting knowledge base whose inclusions are equipped with different probabilities. In particular: let $\kk=\langle \RR, \TT, \AAA \rangle$ be a the KB, where the ABox $\AAA$ is empty, the set of rigid inclusions is, as in the previous case,
    $$\RR=\{\mathit{Fish} \sqsubseteq \forall \mathit{livesIn}.\mathit{Water}\}$$
and the set of typicality properties $\TT$ is now as follows:
\begin{quote}
  \begin{enumerate}
   \item $0.9  \ :: \ \tip(\mathit{Fish}) \sqsubseteq \lnot \mathit{Affectionate}$ 
       \item $0.8 \ :: \ \tip(\mathit{Fish}) \sqsubseteq \mathit{Greyish}$ 
    \item $0.8 \ :: \ \tip(\mathit{Fish}) \sqsubseteq \mathit{Scaly}$ 
    \item $0.9 \ :: \ \tip(\mathit{Fish}) \sqsubseteq \lnot \mathit{Warm}$ 
   \item $0.8 \ :: \ \tip(\mathit{Pet}) \sqsubseteq \forall \mathit{livesIn}.(\lnot \mathit{Water})$
   \item $0.9  \ :: \ \tip(\mathit{Pet}) \sqsubseteq \mathit{Affectionate}$ 
 \item $0.95 \ :: \ \tip(\mathit{Pet}) \sqsubseteq \mathit{Warm}$ 
  \end{enumerate}
\end{quote}
\vspace{-0.1cm}

\noindent As in the previous case, by the properties of the typicality operator $\tip$, we inherit from the underlying $\alct$ that $\tip(\mathit{Pet} \sqcap \mathit{Fish}) \sqsubseteq \forall \mathit{livesIn}.\mathit{Water}$.

In the novel set of generated scenarios we have that the scenario with the highest probability (up to $35 \%$) is both trivial and inconsistent, therefore the following blocks of  scenarios are considered. 

\begin{figure}
\centering
\includegraphics[width=1.2\textwidth]{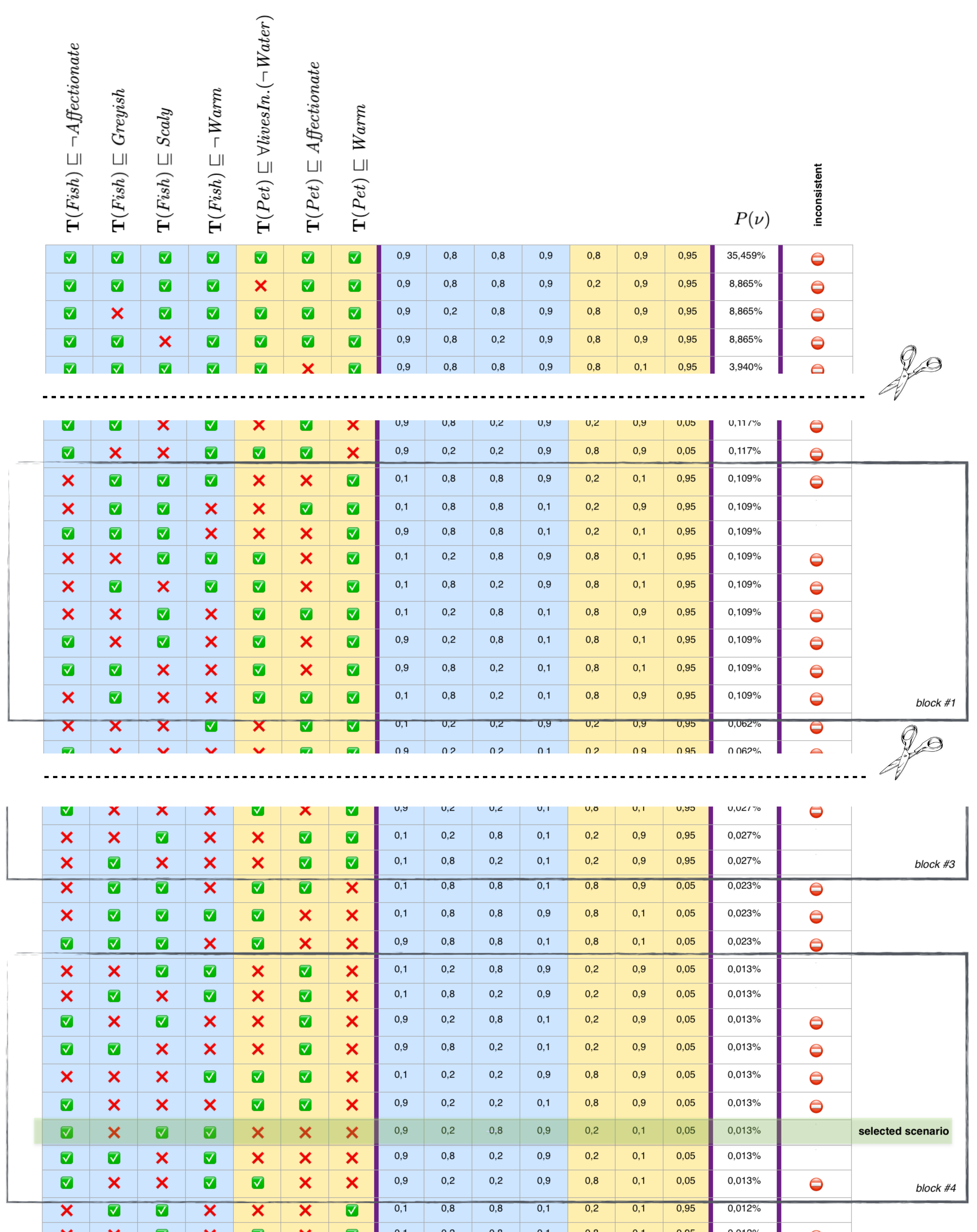}
\vspace{-0.5cm}
\caption{\label{pffigure}Scenarios of the combination of $\mathit{Pet}$ and $\mathit{Fish}$. For the sake of readability we cut off the first part of the figure since the first blocks of the combination, i.e. 43 scenarios, are all inconsistent.}
\end{figure}

Scenarios with probabilities from $35,459\%$ down to $0.117\%$ are inconsistent.
Considering blocks of consistent scenarios  in descending order of their probabilities, all scenarios with probabilities ranging from $0.109\%$ down to $0.027\%$ are 
discarded since they are either trivial, i.e. all the properties of the HEAD are inherited, or they  privilege the MODIFIER against the HEAD.

The first valid block, i.e. block \# 4 in Figure \ref{pffigure}, with probability $0.013\%$, contains four consistent scenarios. Again, two of them are discarded for HEAD/MODIFIER heuristics. Among the remaining two valid scenarios, only one of them is able to model the pet fish phenomenon: it includes three out of four properties of the HEAD, therefore it is not trivial and it is selected by the logic $\pfl$ for the composition of the two initial prototypes. 

In this case $\pfl$ remains agnostic with respect to the final selection of the scenario. For the sake of our purpose, indeed, the fact that  $\pfl$ selects the block containing the scenario with the right answer to the PET FISH problem is, even if a non optimal solution, a satisfying one. In fact, the assigned probabilities in the initial KB can be considered a more complicated testbed for this combination with respect to the previous setup. Indeed, in the latter case, there is i) no probability in the HEAD higher than those in the MODIFIER and ii) the attribute that is expected to be discarded has not a lower probability than the other properties that are assumed to be inherited by the compound concept. 

In this respect, it seems an acceptable compromise that the correct solution is included in the set of selected scenarios having the same probability\footnote{As we will see in the next sections, in application domains in which $\pfl$ is exploited, e.g. in the field of computational creativity, where it is not always easy to define which is the ``correct'' combination, this situation is not uncommon. As a consequence, in such cases the final decision about what selection represents the most appropriate combination, is left to the human decision makers.}. It is worth-noticing that also in the PET FISH formalization proposed by [\cite{lewis2016hierarchical}] it is explicitly mentioned that the result of the right composition in the PET FISH problem is reliant on the values chosen for the prototypes of PET and FISH (as this example shows). 
In particular, with respect to the proposal by [\cite{lewis2016hierarchical}], where the authors assume that the concepts of PET and FISH are exclusively represented by the properties: ``lives in house'' and  ``furry'' (with weighed values of 0.6 and 0.4 respectively) and ``lives in water'' and  ``scaly'' respectively (both with values 0.5), our logic is able to handle a more complicated running example and is able to deal with the problem of blocking - under reasonable circumstances (including some of the most challenging ones) - the inheritance of undesired properties (e.g like in the case of ``greyish'', not presented in the formalization of [\cite{lewis2016hierarchical}] but available in all the descriptions of the PET FISH problem provided by the cognitive science literature on the phenomenon). 
In the case of $\pfl$, the only situation that cannot be handled is represented by a starting KB assigning the highest degree of belief/probability value to the property of the combination that should be discarded. To the best of our knowledge, however, there is no formalization that is currently able to handle such a situation.

\subsection{Linda the feminist bank teller}\label{sezione linda}
We now exploit the logic $\pfl$ in order to tackle the conjunction fallacy problem (or ``Linda Problem''). The problem configuration is as follows: let us suppose to know that Linda is a 31 years old, single, outspoken, and bright lady. She majored in philosophy and was concerned with issues of discrimination and social justice, and also participated in anti-nuclear demonstrations. When asked to rank the probability of the statements 1) ``Linda is a bank teller'' and 2) ``Linda is a bank teller and is active
in the feminist movement'', the majority of people rank 2) as more probable than 1), violating the classic probability rules. In our logic,
let $\kk=\langle \RR, \TT, \AAA \rangle$ be a KB, where 
  $\AAA=\emptyset$,
     $\TT$ is:
     
      \begin{quote}
   $0.8  \ :: \ \tip(\mathit{Feminist}) \sqsubseteq  \mathit{Bright}$  \\
    $0.9  \ :: \ \tip(\mathit{Feminist}) \sqsubseteq  \mathit{OutSpoken}$  \\
    $0.8  \ :: \ \tip(\mathit{Feminist}) \sqsubseteq \exists \mathit{fightsFor}.\mathit{SocialJustice}$  \\
    $0.9  \ :: \ \tip(\mathit{Feminist}) \sqsubseteq \mathit{Environmentalist}$  \\
    $0.6  \ :: \ \scriptsize{\tip(\mathit{BankTeller}) \sqsubseteq \lnot \exists \mathit{fightsFor}.\mathit{SocialJustice}}$  \\
    $0.8  \ :: \ \tip(\mathit{BankTeller}) \sqsubseteq \mathit{Calm}$  
      \end{quote}

\noindent    and   $\RR$ is as follows:

  \begin{quote}
\noindent    $\mathit{BankTeller} \sqsubseteq \exists \mathit{isEmployed}.\mathit{Bank}$  \\
    $\mathit{Feminist} \sqsubseteq \exists \mathit{believesIn}.\mathit{Feminism}$ \\
    $\mathit{Feminist} \sqsubseteq \mathit{Female}$ \\
    $\scriptsize{\mathit{Environmentalist} \sqsubseteq \exists \mathit{isAgainst}.\mathit{NuclearEnergyDevelopment}}$ 
 \end{quote}

\noindent Let us consider the compound concept  $\mathit{Feminist} \sqcap \mathit{BankTeller}.$ It can be obtained in two different ways, namely by choosing $\mathit{Feminist}$ as the HEAD and $\mathit{BankTeller}$ as the MODIFIER, or vice versa. First, we consider $\mathit{Feminist}$ as the HEAD.
In $\pfl$, the compound concept inherits all the rigid properties, that is to say $\mathit{Feminist} \sqcap \mathit{BankTeller}$ is included in $\exists \mathit{isEmployed}.\mathit{Bank}$, in $\exists \mathit{believesIn}.\mathit{Feminism}$ and in $\mathit{Female}$.
Concerning the typical properties, two of them are in contrast, namely typical feminists fight for social justice, whereas typical bank tellers do not. All the scenarios including both $\tip(\mathit{Feminist}) \sqsubseteq \exists \mathit{fightsFor}.\mathit{SocialJustice}$ and $\tip(\mathit{BankTeller}) \sqsubseteq \lnot \exists \mathit{fightsFor}.\mathit{SocialJustice}$ are then inconsistent. Concerning the remaining ones, 
  scenarios including $\tip(\mathit{BankTeller}) \sqsubseteq \lnot \exists \mathit{fightsFor}.\mathit{SocialJustice}$ are discarded, in favor of scenarios including $\tip(\mathit{Feminist}) \sqsubseteq \exists \mathit{fightsFor}.\mathit{SocialJustice}$. The most obvious scenario, with the highest probability,  corresponds to the one including all the  typicality inclusions related to the HEAD. In the logic $\pfl$ we discard it and we focus on the remaining ones. Among them, one of the scenarios having the highest probability is the one not including $\tip(\mathit{Feminist}) \sqsubseteq \mathit{Bright}$. This scenario defines the following $\mathit{Feminist} \sqcap \mathit{BankTeller}$-revised knowledge base:

 \begin{quote}
  $0.9 \ :: \ \tip(\mathit{Feminist} \sqcap \mathit{BankTeller}) \sqsubseteq  \mathit{OutSpoken}$  \\
    $0.8 \ :: \ \scriptsize{\tip(\mathit{Feminist} \sqcap \mathit{BankTeller}) \sqsubseteq \exists \mathit{fightsFor}.\mathit{SocialJustice}}$  \\
    $0.9 \ :: \ \scriptsize{\tip(\mathit{Feminist} \sqcap \mathit{BankTeller}) \sqsubseteq \mathit{Environmentalist}}$ \\
    $0.8 \ :: \ \tip(\mathit{Feminist} \sqcap \mathit{BankTeller}) \sqsubseteq  \mathit{Calm}$ 
 \end{quote}

\noindent Let us now consider the case of the instance Linda, that is described as follows:

\begin{quote}
   $\mathit{YoungWoman}(\mathit{linda})$\\
   $\exists \mathit{graduatedIn}.\mathit{Philosophy}(\mathit{linda})$\\
   $\mathit{Outspoken}(\mathit{linda})$\\
   $\mathit{Bright}(\mathit{linda})$\\
   $\mathit{Single}(\mathit{linda})$\\
   $\exists \mathit{fightsFor}.\mathit{SocialJustice}(\mathit{linda})$\\
   $\exists \mathit{isAgainst}.\mathit{NuclearEnergyDevelopment}(\mathit{linda})$
 \end{quote}
In our logic, solving the conjunction fallacy problem \footnote{The attempt of modelling a reasoning error could seem \emph{prima facie}, counterintuitive in an AI setting. However, it is worth-noting that this type of reasoning corresponds to a very powerful evolutionary heuristics developed by humans and strongly relying on common-sense knowledge. The use of typical knowledge in cognitive tasks, in fact,
has to do with the constraints that concern every finite agent
that has a limited access to the knowledge relevant for a given task. Consider for example the following variant of the Linda problem. Let
us suppose that a certain individual Pluto is described as follows. He
weighs about 250 kg, and he is approximately two meters tall. His
body is covered with a thick, dark fur, he has a large mouth with robust
teeth and paws with long claws. He roars and growls. Now, given this
information, we have to evaluate the plausibility of the two following
alternatives:
a) Pippo is a mammal;
b) Pippo is a mammal, and he is wild and dangerous.
Which is the “correct” answer? According to the dictates of the normative
theory of probability, it is surely a). But if you encounter Pippo
in the wilderness, it would probably be best to run.} means that we have to find the most appropriate category for Linda. In our case the choice is between $\mathit{BankTeller}$ and $\mathit{Feminist} \sqcap \mathit{BankTeller}$. 
We can assume that, in absence of any other information, the described properties that are explicitly assigned to the instance Linda can be set to a default probability value of 0.6 (that is to say: the asserted properties about Linda are considerer ``typical enough'' for her description).
Let us first consider the $\mathit{Feminist} \sqcap \mathit{BankTeller}$-revised knowledge base, with an ABox asserting that Linda is a bank teller, that is to say $$\AAA_1=\{\mathit{BankTeller}(\mathit{linda})\},$$ and let us consider each property of the instance Linda and the associated probability of entailment. Observe that none of such properties are entailed by the $\mathit{Feminist} \sqcap \mathit{BankTeller}$-revised knowledge base with $\AAA_1$, therefore, for each property of the form $D(\mathit{linda})$ we have that $\PP(D(\mathit{linda}),0.6)=0$ (by Definition \ref{probability query}).
On the other hand, let us consider an ABox asserting that Linda is a feminist bank teller, namely $$\AAA_2=\{(\mathit{Feminist} \sqcap \mathit{BankTeller})(\mathit{linda})\}.$$ In this case, we have that:
\begin{itemize}
\item $\mathit{YoungWoman}(\mathit{linda})$  is not entailed from the $\mathit{Feminist} \ \sqcap \ \mathit{BankTeller}$-revised knowledge base, therefore \\ $\PP(\mathit{YoungWoman}(\mathit{linda}),$ $0.6)=0;$
   the same for $\exists \mathit{graduatedIn}.\mathit{Philosophy}(\mathit{linda})$ and $\mathit{Single}(\mathit{linda})$;
 \item $\mathit{Outspoken}(\mathit{linda})$ is entailed from the $\mathit{Feminist} \sqcap \mathit{BankTeller}$-revised knowledge base with $\AAA_2$, then, by Definition \ref{probability query},  we have  $\PP(\mathit{Outspoken}(\mathit{linda}),0.6)= 0.6 \times 0.9 = 0.54$, where $0.9$ is the probability of  $\tip(\mathit{Feminist} \sqcap \mathit{BankTeller}) \sqsubseteq  \mathit{OutSpoken}$ in the $\mathit{Feminist} \sqcap \mathit{BankTeller}$-revised KB;
\item the same holds for  $\exists \mathit{fightsFor}.\mathit{SocialJustice}$ $(\mathit{linda})$, which is entailed by using $\AAA_2$: in this case, we have that \\ $\PP(\exists \mathit{fightsFor}.\mathit{SocialJustice} (\mathit{linda}),$ $0.6)=0.6 \times 0.8 = 0.48$;
 \item the fact $\exists \mathit{isAgainst}.\mathit{NuclearEnergyDevelopment}(\mathit{linda})$ is entailed by using $\AAA_2$. Observe that $\mathit{Environmentalist} \sqsubseteq \mathit{NuclearEnergyDevelopment}$ follows from $\RR$ in  standard $\mathcal{ALC}$, then $\PP (\exists \mathit{isAgainst}.\mathit{NuclearEnergyDevelopment}$ $(\mathit{linda}),0.6) = 0.6 \times 0.9 = 0.54$ by Definition \ref{probability query}, where $0.9$ is the probability of $\tip(\mathit{Feminist} \sqcap \mathit{BankTeller}) \sqsubseteq \mathit{Environmentalist}$ in the $\mathit{Feminist} \sqcap \mathit{BankTeller}$-revised KB.
 \end{itemize}
Computing the sum of the probabilities of the queries of all facts about Linda,  we obtain $0.54 + 0.48  + 0.54 = 1.56$, to witness that the choice of $\AAA_2$ is more appropriate w.r.t. the choice of $\AAA_1$ where the sum is $0$. This means that, in our logic, the human choice of classifying Linda as a feminist bank teller sounds perfectly plausible
and has to be preferred to the alternative one of classifying her as a bank teller.

Let us now consider the case in which $\mathit{BankTeller}$ is the HEAD. In this case, the $\mathit{BankTeller} \sqcap \mathit{Feminist}$-revised knowledge base would be as follows:

 \begin{quote}
  $0.8 \ :: \ \tip(\mathit{BankTeller} \sqcap \mathit{Feminist}) \sqsubseteq  \mathit{Bright}$  \\
  $0.9 \ :: \ \tip(\mathit{BankTeller} \sqcap \mathit{Feminist}) \sqsubseteq  \mathit{OutSpoken}$  \\
    $0.8 \ :: \ \scriptsize{\tip(\mathit{BankTeller} \sqcap \mathit{Feminist}) \sqsubseteq \exists \mathit{fightsFor}.\mathit{SocialJustice}}$  \\
    $0.9 \ :: \ \scriptsize{\tip(\mathit{BankTeller} \sqcap \mathit{Feminist}) \sqsubseteq \mathit{Environmentalist}}$ \\
    $0.8 \ :: \ \tip(\mathit{BankTeller} \sqcap \mathit{Feminist}) \sqsubseteq  \mathit{Calm}$ 
 \end{quote}

\noindent Also in this case, the probability that Linda is a bank teller feminist is higher (2.04) than the probability of classifying her as a bank teller (as in the previous case, equals to zero).

\subsection{Metaphorical Concept Combination: a prototype of the Stone Lion}\label{sezione metafora}
In this section we consider a particular type of concept combination based on the nonmononotonic character of $\pfl$. In particular, we take into account a classical case of metaphorical concept composition considered in the field of cognitive semantics: the Stone Lion example [\cite{gardenfors2014geometry,gardenfors1998concept,franks1995sense}]. 
If we consider the concept \emph{Lion} in isolation, typically it is inferred that it is alive, it has fur and a tail, and so on. If we consider the combination of concept \emph{Stone} and \emph{Lion}, on the other hand, the only inherited aspect which is Lion-like is its shape, this means that, in this case, the effect of the combination is obtained because a stone object is metaphorically seen as a lion, due to its shape.  
Let us consider in detail  this example, and let us exploit the logic $\pfl$ in order to provide a commonsense description of a prototype of the Stone Lion. 
Let $\kk=\langle \RR, \TT, \AAA \rangle$ be a KB, where: $\AAA$ is empty, $\RR=\{\mathit{MainColorYellowish} \sqcap \mathit{MainColorGreyish} \sqsubseteq \bot \}$ are empty, and $\TT$ is as follows:

\begin{quote}
  \begin{enumerate}
   \item $0.9 \ :: \ \tip(\mathit{Stone}) \sqsubseteq \mathit{HardMaterial}$
   \item $0.8  \ :: \ \tip(\mathit{Stone}) \sqsubseteq \mathit{MainColorGreyish}$ 
      \item $0.7  \ :: \ \tip(\mathit{Stone}) \sqsubseteq \mathit{Rolling}$ 
   \item $0.8  \ :: \ \tip(\mathit{Lion}) \sqsubseteq  \mathit{MainColorYellowish}$ 
   \item $0.7 \ :: \ \tip(\mathit{Lion}) \sqsubseteq \exists \mathit{has}.\mathit{Tail}$
  \end{enumerate}
\end{quote}

\noindent  We consider $\mathit{Stone}$ as the HEAD and $\mathit{Lion}$ as the MODIFIER in defining this combination. Figure \ref{stonelion} shows the $2^5 = 32$ different scenarios, one row for each scenario, in descending order of provability. As in the example of the pet fish, inconsistent scenarios are highlighted in the last column.

\begin{figure}
\centering
\includegraphics[width=\textwidth]{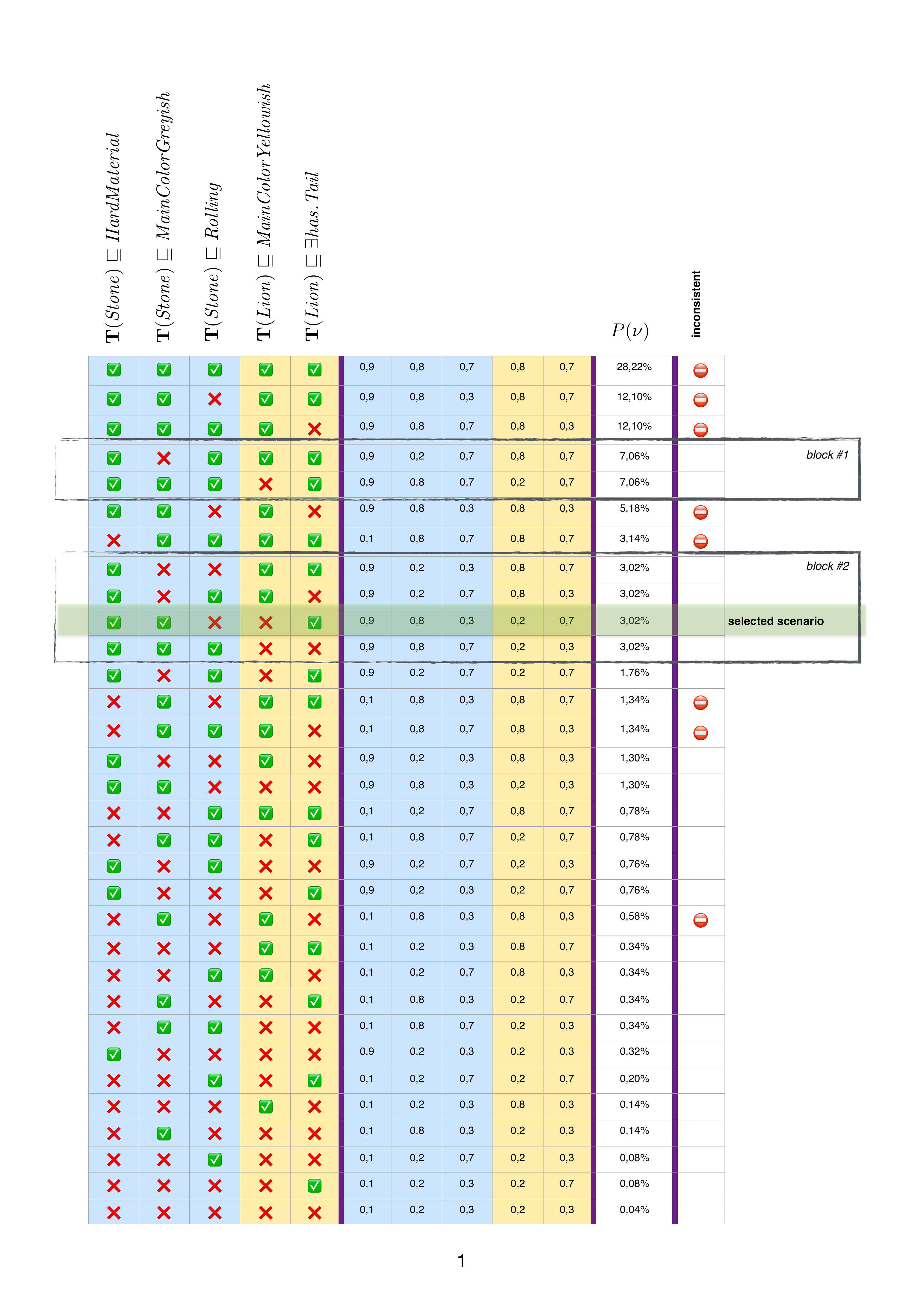}
\vspace{-0.6cm}
\caption{\label{stonelion}Scenarios of the combination of $\mathit{Stone}$ and $\mathit{Lion}$.}
\end{figure}

From a probabilistic perspective, the first three scenarios are inconsistent since both inclusions 2 and 4 are included, in other words we would have $\tip(\mathit{Stone} \sqcap \mathit{Lion}) \sqsubseteq \mathit{MainColorGreyish}$ and $\tip(\mathit{Stone} \sqcap \mathit{Lion}) \sqsubseteq \mathit{MainColorYellowish}$, having that $\mathit{MainColorGreyish} \sqcap \mathit{MainColorYellowish} \sqsubseteq \bot$. 
The first block to  consider contains two scenarios with probability $7.056\%$: the first one is discarded since it privileges a property of the MODIFIER (having $\mathit{MainColorYellowish}$) rather than a contrasting one of the HEAD (having $\mathit{MainColorGreyish}$). The second one is discarded, since it allows to inherit all the properties of the HEAD and its then considered as trivial. The following useful block is defined by four scenarios having probability $3.02\%$: 
\begin{itemize}
   \item the first and the second ones in Figure \ref{stonelion} are discarded since, again, the property $\mathit{MainColorYellowish}$ of the MODIFIER is preferred to the contrasting one of the HEAD;
   \item the last one is discarded since it contains all the properties of the HEAD, then it is trivial.
\end{itemize}

The remaining scenario, the third one of the block in Figure \ref{stonelion}, is the one selected by the logic $\pfl$. It contains inclusions 1 and 2 from the HEAD and inclusion 5 from the modifier, obtaining the following $\mathit{Stone} \sqcap \mathit{Lion}$-revised knowledge base:

  \begin{itemize}
   \item[1.] $0.9 \ :: \ \tip(\mathit{Stone} \sqcap \mathit{Lion}) \sqsubseteq \mathit{HardMaterial}$
   \item[2.] $0.8  \ :: \ \tip(\mathit{Stone} \sqcap \mathit{Lion}) \sqsubseteq \mathit{MainColorGreyish}$ 
   \item[5.] $0.7 \ :: \ \tip(\mathit{Stone} \sqcap \mathit{Lion}) \sqsubseteq \exists \mathit{has}.\mathit{Tail}$
  \end{itemize}

Notice that, if $\RR=\{\mathit{Stone} \sqsubseteq \lnot \mathit{Breath}\}$ and $\TT$ also contains $\tip(\mathit{Lion}) \sqsubseteq \mathit{Breath}$, in our logic we would also infer that $\mathit{Stone} \sqcap \mathit{Lion} \sqsubseteq \lnot \mathit{Breath}$.

\section{Artificial Prototypes Composition and Concept Invention}\label{sez:novel}

In this section we exploit the logic $\pfl$ to show both i) how it allows to automatically generate novel, plausible, prototypical concepts by composing two initial prototypes and ii) how it  can be used as a generative tool in the field of computational creativity (with applications in the so called creative industry). In detail, we extend preliminary results presented by [\cite{lietopozzato2019creativityia}] and we first show how our logic can model the generation of a quite complex concept recently introduced in the field of narratology, i.e. that one of the ANTI-HERO (a role invented by narratologists to generate new story lines), by combining the typical properties of the concepts HERO and VILLAIN. Of course, the specific domain of the example is not relevant here; our goal is showing how $\pfl$ can model this kind of prototypical concept composition (a crucial aspect of human concept invention) that, on the other hand, has been proven to be problematic for other kinds of logics (e.g fuzzy logic, [\cite{osherson1981adequacy,hampton2011conceptual})]. We then show how the same machinery can be used as a creativity support tool to generate a new type of villain for a video game or a movie.

\subsection{Anti Hero}\label{sec:antieroe}
We will take into account the concepts of HERO, ANTI-HERO and VILLAIN
extracted by the common sense descriptions coming from the TvTropes repository \footnote{\url{https://tvtropes.org}}. In such online repository, typical descriptions of character roles are provided. They  can be useful for practitioners
of the narrative field in order to design their own character according to the main assets
presented in such schemas. In particular, Tropes can be seen as devices and conventions
that a writer can reasonably rely on as being present in the audience members’ minds and
expectations. Regarding the HERO, TvTropes identifies the following relevant representative
features: e.g. the fact that it is characterized by his/her fights against the VILLAIN of a
story, the fact that his/her actions are necessarily guided by general goals to be achieved in
the interest of the collectivity, the fact that they fight against the VILLAIN in a fair way
and so on. Examples of such Trope are: Superman, Flash Gordon etc.. The ANTI-HERO,
on the other hand, is described as characterized by the fact of sharing most of its typical
traits with the HERO (e.g. the fact that it is the protagonist of a plot fighting against the
VILLAIN of the story); however, his/her moves are not guided by a general spirit of sacrifice
for the collectivity but, rather, they are usually based on some personal motivations that,
incidentally and/or indirectly, coincide with the needs of the collectivity. Furthermore the
ANTI-HERO may also act in a not fair way in order to achieve the desired goal. A classical
example of such trope is Batman, whose moves are guided by his desire of revenge. Finally
the VILLAIN is represented as a classic negative role in a plot and is characterized as the main opponent of the protagonist/HERO.
In addition to this classical contraposition, TvTropes also reports some physical elements
characterizing such role from a visual point of view. For example: the characters of this
Trope are usually physically endowed with some demoniac cues (e.g. they have the ``eyes of
fire''). Finally, they are guided by negative moral values. Examples of such role can be easily
taken from the classical literature to the modern comics. Some representative exemplars are
Cruella de Vil in Disney’s filmic saga or Voldemort in Harry Potter.

Let us now exploit our logic $\pfl$ in order to define a prototype of ANTI-HERO. First of all, we define a knowledge base describing both rigid and typical properties of concepts HERO and VILLAIN, then we rely on the logic $\pfl$ in order to formalize an $\mathit{AntiHero}$-revised knowledge base.

Let $\kk=\langle \RR, \TT, \AAA \rangle$ be a KB, where the ABox $\AAA$ is empty. Concerning rigid properties, let $\RR$ be as follows:
   \begin{enumerate}
       \item[R1] $\mathit{Hero} \sqsubseteq \exists \mathit{hasOpponent}.\mathit{Villain}$
       \item[R2] $\mathit{Villain} \sqsubseteq \forall \mathit{fightsFor}.\mathit{PersonalGoal}$
       \item[R3] $\mathit{Villain} \sqsubseteq \mathit{WithNegativeMoralValues}$
     \item[R4] $\mathit{CollectiveGoal} \sqcap \mathit{PersonalGoal} \sqsubseteq \bot$ 
      \item[R5] $\mathit{WithPositiveMoralValues} \sqcap \mathit{WithNegativeMoralValues} \sqsubseteq \bot$ 
     \item[R6] $\mathit{AngelicIconicity} \sqcap \mathit{DemoniacIconicity} \sqsubseteq \bot$ 
   \end{enumerate}
   
 Prototypical properties of villains and heroes are described in $\TT$ is as follows:
  \begin{enumerate}
   \item[T1] $0.95 \ :: \ \tip(\mathit{Hero}) \sqsubseteq \mathit{Protagonist}$
   \item[T2] $0.85  \ :: \ \tip(\mathit{Hero}) \sqsubseteq \exists \mathit{fightsFor}.\mathit{CollectiveGoal}$ 
   \item[T3] $0.9 \ :: \ \tip(\mathit{Hero}) \sqsubseteq \mathit{WithPositiveMoralValues}$
   \item[T4] $0.6 \ :: \ \tip(\mathit{Hero}) \sqsubseteq \mathit{AngelicIconicity}$
   \item[T5] $0.75 \ :: \ \tip(\mathit{Villain}) \sqsubseteq \mathit{DemoniacIconicity}$
        \item[T6] $0.8 \ :: \ \tip(\mathit{Villain}) \sqsubseteq  \mathit{Implulsive}$
      \item[T7] $0.75 \ :: \ \tip(\mathit{Villain}) \sqsubseteq  \mathit{Protagonist}$
   \end{enumerate}


We make use of the logic $\pfl$ in order to build the compound concept $\mathit{AntiHero}$ as the result of the combination of concepts $\mathit{Hero}$ and $\mathit{Villain}$. Differently from what the natural language seems to suggest, we consider this compound concept by assuming that the HEAD is $\mathit{Villain}$ (since the ANTI-HERO shares more typical traits with this concept than with the HERO concept). 

First of all, we have that the compound concepts inherits all the rigid properties of both its components (if not contradictory), therefore in the logic $\pfl$ we have that:

\begin{enumerate}
        \item[(i)] $\mathit{AntiHero} \sqsubseteq \exists \mathit{hasOpponent}.\mathit{Villain}$
        \item[(ii)] $\mathit{AntiHero} \sqsubseteq \forall \mathit{fightsFor}.\mathit{PersonalGoal}$
        \item[(iii)] $\mathit{AntiHero} \sqsubseteq  \mathit{WithNegativeMoralValues}$
\end{enumerate}

For the typical properties, we consider all the $2^7=256$ different scenarios obtained from all possible selections about inclusion in $\TT$. Some of them are inconsistent, namely those including either axiom T2 or axiom T3, since they would ascribe properties in contrast with inherited rigid properties of (ii) and (iii): rigid properties impose that an anti hero has negative moral values, and all his goals are personal, therefore he is an atypical hero in those respects (T2 states that typical heroes fights also for some collective goals, whereas T3 states that normally heroes have positive moral values). Also scenarios containing both axioms T4 and T5 are inconsistent, since the fact that the concepts $\mathit{AngelicIconicity}$ and $\mathit{DemoniacIconicity}$ are disjoint (formalized by R6).

Let us consider the remaining, consistent scenarios: the one having the highest probability considers all the properties of both concepts by excluding only $\mathit{AngelicIconicity}$, that is to say the one with the lowest probability between the two properties in conflict. In $\pfl$  this scenario is discarded since it is the most trivial one. 
When we consider scenarios less trivial, i.e., more surprising scenarios (we analyze scenarios in decreasing order of probability), we discard the scenario with probability $0.13\%$, which includes 
  T4, associated to the MODIFIER, rather than T5, associated to the HEAD, allowing to conclude, in a counter intuitive way, that typical anti heroes have an angelic iconicity rather than a demoniac one.

Next scenarios, sharing the same probability ($0.09\%$), are as follows:
\[
\begin{array}{l|cr}
\hline\\
\begin{minipage}{8cm}
 \begin{itemize}
   \item[T1] $0.95 \ :: \ \tip(\mathit{Hero}) \sqsubseteq \mathit{Protagonist}$
   \item[T5] $0.75 \ :: \ \tip(\mathit{Villain}) \sqsubseteq \mathit{DemoniacIconicity}$
        \item[T6] $0.8 \ :: \ \tip(\mathit{Villain}) \sqsubseteq  \mathit{Implulsive}$
\end{itemize}
\end{minipage}
& \ &
\begin{minipage}{7cm}
\begin{itemize}
   \item[T1] $0.95 \ :: \ \tip(\mathit{Hero}) \sqsubseteq \mathit{Protagonist}$
        \item[T6] $0.8 \ :: \ \tip(\mathit{Villain}) \sqsubseteq  \mathit{Implulsive}$
      \item[T7] $0.75 \ :: \ \tip(\mathit{Villain}) \sqsubseteq  \mathit{Protagonist}$
\end{itemize}
\end{minipage} \\ \\
\hline
\end{array}
\]

\noindent According to the logic $\pfl$, both are adequate and represent the outcome of the whole heuristic procedures adopted in $\pfl$. Probably, in this case, it could be more useful to opt for the solution on the left allowing to inherit a further property (i.e. $\mathit{DemoniacIconicity}$) for the generated prototypical Anti-Hero. However, we remain agnostic about the selection of the  final options provided by $\pfl$. This choice can be plausibly left to human decision makers and based on their own goals.

\subsection{Generating a Novel Character via $\pfl$: a Villain Chair}
Let us now exploit our logic $\pfl$ in order to create a new compound concept: e.g. a new type of villain for a video game or a movie, obtained by as the combination of concepts $\mathit{Villain}$ (as HEAD) and $\mathit{Chair}$ (as MODIFIER).
Let $\kk=\langle \RR, \TT, \AAA \rangle$ be a KB, where: $\AAA$ is empty, $\RR$ is as follows:
   \begin{enumerate}
       \item[R1] $\mathit{Villain} \sqsubseteq \exists \mathit{fightsFor}.\mathit{PersonalGoal}$ 
       \item[R2] $\mathit{Villain} \sqsubseteq \mathit{Animate}$
       \item[R3] $\mathit{Villain} \sqsubseteq \mathit{WithNegativeMoralValues}$
              \item[R4] $\mathit{Chair} \sqsubseteq  \exists \mathit{hasComponent}.\mathit{SupportingSeatComponent}$
              \item[R5] $\mathit{Chair} \sqsubseteq \exists \mathit{hasComponent}.\mathit{Seat}$
     \item[R6] $\mathit{CollectiveGoal} \sqcap \mathit{PersonalGoal} \sqsubseteq \bot$ 
   \end{enumerate}
and $\TT$ is as follows:
  \begin{enumerate}
   \item[T1] $0.9 \ :: \ \scriptsize{\tip(\mathit{Villain}) \sqsubseteq \mathit{DemoniacIconicity}}$
        \item[T2] $0.75 \ :: \ \tip(\mathit{Villain}) \sqsubseteq \exists \mathit{hasOpponent}.\mathit{Hero}$
       \item[T3] $0.75 \ :: \ \tip(\mathit{Villain}) \sqsubseteq \mathit{Protagonist}$
     \item[T4] $0.8 \ :: \ \tip(\mathit{Villain}) \sqsubseteq  \mathit{Impulsive}$
       \item[T5] $0.95 \ :: \ \tip(\mathit{Chair}) \sqsubseteq \lnot \mathit{Animate}$
             \item[T6] $0.95 \ :: \ \tip(\mathit{Chair}) \sqsubseteq \exists \mathit{hasComponent}.\mathit{Back}$
       \item[T7] $0.65 \ :: \ \tip(\mathit{Chair}) \sqsubseteq \exists \mathit{madeOf}.\mathit{Wood}$
 \item[T8] $0.8 \ :: \ \tip(\mathit{Chair}) \sqsubseteq \mathit{Comfortable}$
       \item[T9] $0.7 \ :: \ \tip(\mathit{Chair}) \sqsubseteq \mathit{Inflammable}$
   \end{enumerate}

\noindent We consider the 512 scenarios, from which we discard the inconsistent ones, namely those including T5: indeed, since R2 imposes that villains are animate, in the underlying $\alct$ we conclude that $\mathit{Villain} \sqcap \mathit{Chair} \sqsubseteq \mathit{Animate}$, therefore all scenarios including T5, imposing that $\mathit{Villain} \sqcap \mathit{Chair} \sqsubseteq \lnot \mathit{Animate}$ are inconsistent.
We also discard the most obvious scenario including all the typicality inclusions of $\RR$, having probability of $14\%$. We also discard 
the following trivial scenarios containing all the inclusions related to the HEAD, namely T1, T2, T3, and T4.

The first suitable block according to Algorithm \ref{algoritmo} is the one whose scenarios have probability $4.67\%$ and contain all properties coming from the MODIFIER and three out of four properties coming from the HEAD. Such scenarios, defining two alternative revised knowledge bases (one containing T2 and not T3, the other one containing T3 and not T2), are as follows:

\begin{footnotesize}
\[
\begin{array}{l|cr}
\hline\\
\begin{minipage}{7cm}
 \begin{itemize}
   \item[T1] $0.9 \ :: \ \scriptsize{\tip(\mathit{Villain}) \sqsubseteq \mathit{DemoniacIconicity}}$
       \item[T3] $0.75 \ :: \ \tip(\mathit{Villain}) \sqsubseteq \mathit{Protagonist}$
     \item[T4] $0.8 \ :: \ \tip(\mathit{Villain}) \sqsubseteq  \mathit{Impulsive}$
             \item[T6] $0.95 \ :: \ \tip(\mathit{Chair}) \sqsubseteq \exists \mathit{hasComponent}.\mathit{Back}$
       \item[T7] $0.65 \ :: \ \tip(\mathit{Chair}) \sqsubseteq \exists \mathit{madeOf}.\mathit{Wood}$
 \item[T8] $0.8 \ :: \ \tip(\mathit{Chair}) \sqsubseteq \mathit{Comfortable}$
       \item[T9] $0.7 \ :: \ \tip(\mathit{Chair}) \sqsubseteq \mathit{Inflammable}$

\end{itemize}
\end{minipage}
& \ &
\begin{minipage}{7cm}
\begin{itemize}
   \item[T1] $0.9 \ :: \ \scriptsize{\tip(\mathit{Villain}) \sqsubseteq \mathit{DemoniacIconicity}}$
        \item[T2] $0.75 \ :: \ \tip(\mathit{Villain}) \sqsubseteq \exists \mathit{hasOpponent}.\mathit{Hero}$
     \item[T4] $0.8 \ :: \ \tip(\mathit{Villain}) \sqsubseteq  \mathit{Impulsive}$
             \item[T6] $0.95 \ :: \ \tip(\mathit{Chair}) \sqsubseteq \exists\mathit{hasComponent}.\mathit{Back}$
       \item[T7] $0.65 \ :: \ \tip(\mathit{Chair}) \sqsubseteq \exists\mathit{madeOf}.\mathit{Wood}$
 \item[T8] $0.8 \ :: \ \tip(\mathit{Chair}) \sqsubseteq \mathit{Comfortable}$
       \item[T9] $0.7 \ :: \ \tip(\mathit{Chair}) \sqsubseteq \mathit{Inflammable}$
\end{itemize}
\end{minipage} \\ \\
\hline
\end{array}
\]
\end{footnotesize}

\noindent According to Algorithm \ref{algoritmo}, both proposals  neither are trivial (not all properties of the HEAD are ascribed to the combined concept) nor they give preference to the MODIFIER with respect to the HEAD for conflicting typical properties (in the example, we have only a conflict between a rigid and typical property, the above mentioned R2 and T5).
 These scenarios are the preferred ones selected by the logic $\pfl$.

However, in this application setting, we could imagine to use our framework as a creativity support tool and thus considering alternative - more surprising - scenarios by adding additional constraints.
For example, we could impose that the compound concept should inherit exactly six properties. In this case, we would get that the scenario having the highest probability ($3.2 \%$) is the one including 
all the properties of the HEAD, namely T1, T2, T3 and T4, and two out of four properties of the MODIFIER, namely T6 and T8. Due to its triviality, this scenario is discarded, in favor of the following block of two scenarios (probability $2.51\%$), obtained by excluding T7 of the MODIFIER and one out of four properties of the HEAD: 


\begin{footnotesize}
\[
\begin{array}{l|cr}
\hline\\
\begin{minipage}{8cm}
 \begin{itemize}
   \item[T1] $0.9 \ :: \ \scriptsize{\tip(\mathit{Villain} \sqcap \mathit{Chair}) \sqsubseteq \mathit{DemoniacIconicity}}$
       \item[T3] $0.75 \ :: \ \tip(\mathit{Villain} \sqcap \mathit{Chair}) \sqsubseteq \mathit{Protagonist}$
     \item[T4] $0.8 \ :: \ \tip(\mathit{Villain} \sqcap \mathit{Chair}) \sqsubseteq  \mathit{Impulsive}$
             \item[T6] $0.95 \ :: \ \tip(\mathit{Villain} \sqcap \mathit{Chair}) \sqsubseteq \exists \mathit{hasComponent}.\mathit{Back}$
 \item[T8] $0.8 \ :: \ \tip(\mathit{Villain} \sqcap \mathit{Chair}) \sqsubseteq \mathit{Comfortable}$
       \item[T9] $0.7 \ :: \ \tip(\mathit{Villain} \sqcap \mathit{Chair}) \sqsubseteq \mathit{Inflammable}$

\end{itemize}
\end{minipage}
& \ &
\begin{minipage}{8cm}
\begin{itemize}
   \item[T1] $0.9 \ :: \ \scriptsize{\tip(\mathit{Villain} \sqcap \mathit{Chair}) \sqsubseteq \mathit{DemoniacIconicity}}$
        \item[T2] $0.75 \ :: \ \tip(\mathit{Villain} \sqcap \mathit{Chair}) \sqsubseteq \exists \mathit{hasOpponent}.\mathit{Hero}$
     \item[T4] $0.8 \ :: \ \tip(\mathit{Villain} \sqcap \mathit{Chair}) \sqsubseteq  \mathit{Impulsive}$
             \item[T6] $0.95 \ :: \ \tip(\mathit{Villain} \sqcap \mathit{Chair}) \sqsubseteq \exists \mathit{hasComponent}.\mathit{Back}$
  \item[T8] $0.8 \ :: \ \tip(\mathit{Villain} \sqcap \mathit{Chair}) \sqsubseteq \mathit{Comfortable}$
       \item[T9] $0.7 \ :: \ \tip(\mathit{Villain} \sqcap \mathit{Chair}) \sqsubseteq \mathit{Inflammable}$
\end{itemize}
\end{minipage} \\ \\
\hline
\end{array}
\]
\end{footnotesize}

These are the scenarios selected by $\pfl$ and are those providing plausible but not obvious, creative definitions of a villain chair.

\subsection{Generating Compounds by Combining Multiple Concepts: a prototype for the Chimera}\label{sec:chimera}
In this section we show how the mechanism of the logic $\pfl$ can be used in order to combine more than two atomic concepts (a case never taken into account in previous formalizations, e.g. in [\cite{lewis2016hierarchical,eppe2018computational}] etc.). Such examples show how the mechanisms behind $\pfl$ can be iteratively applied without loss of efficacy in the produced output. In our opinion, this is a symptom of the fact that our framework is able to actually capture some foundational elements of common-sense conceptual compositionality.

Let us consider the example of a chimera, a mythological hybrid entity composed of the parts of more than one animal. It is usually depicted as a lion, with the head of a goat positioned in the center of its body, and a tail ended with the head of a fire-breathing dragon \footnote{Technically this kind of combination is called conceptual blending and is slightly different from a classical combination since the generated concept is an entirely new one and is not a subset of classes generating it, see [\cite{nagai2006formaldescription}] for more details on the subtle differences between the two tasks.}.

First of all, let us describe the three atomic concepts to be combined, namely $\mathit{Lion}$, $\mathit{Goat}$, and $\mathit{Dragon}$.
Let $\kk$ be as follows:

\begin{enumerate}
   \item $\mathit{Lion} \sqsubseteq  \mathit{Animal}$
   \item $0.7  \ :: \ \tip(\mathit{Lion}) \sqsubseteq  \mathit{MainColorYellowish}$ 
   \item $0.9 \ :: \ \tip(\mathit{Lion}) \sqsubseteq \exists \mathit{has}.\mathit{Tail}$ 
   \item $0.8 \ :: \ \tip(\mathit{Lion}) \sqsubseteq \exists \mathit{has}.\mathit{Mane}$ 
   \item $\mathit{Goat} \sqsubseteq  \mathit{Animal}$
   \item $0.7  \ :: \ \tip(\mathit{Goat}) \sqsubseteq \mathit{MainColorWhitish}$ 
   \item $0.75 \ :: \ \tip(\mathit{Goat}) \sqsubseteq \exists \mathit{has}.\mathit{Tail}$ 
   \item $0.9 \ :: \ \tip(\mathit{Goat}) \sqsubseteq \exists \mathit{has}.\mathit{Horns}$ 
   \item $0.9 \ :: \ \tip(\mathit{Goat}) \sqsubseteq \exists \mathit{has}.\mathit{Beards}$ 
   \item $\mathit{Dragon} \sqsubseteq  \mathit{LegendaryCreature}$
   \item $0.9  \ :: \ \tip(\mathit{Dragon}) \sqsubseteq  \mathit{FireBreathing}$ 
   \item $0.8 \ :: \ \tip(\mathit{Dragon}) \sqsubseteq  \mathit{Fly}$    
   \item $0.9 \ :: \ \tip(\mathit{Dragon}) \sqsubseteq  \mathit{Aggressive}$   
   \item $\mathit{Whitish} \sqcap \mathit{Yellowish} \sqsubseteq \bot$
\end{enumerate}

We consider that $\mathit{Lion}$ as the HEAD, whereas both $\mathit{Goat}$ and $\mathit{Dragon}$ are modifiers.

We discard all the inconsistent scenarios including both 2 and 6. We also discard  the most trivial scenario including all the typicality inclusions of the HEAD, whose probability is $5.95\%$ and, in the same block of scenarios, we also discard the one giving preference to the MODIFIER $\mathit{Goat}$ with respect to the HEAD. More precisely the scenario preferring the inclusion 6 to 2. Among the remaining ones, the selected scenario having the highest probability ($2.55\%$) is as follows, not including both 2 and 6 ($Chimera$-revised knowledge base):

\begin{itemize}
   \item $\mathit{Chimera} \sqsubseteq  \mathit{Animal}$
   \item $0.9 \ :: \ \tip(\mathit{Chimera}) \sqsubseteq \exists \mathit{has}.\mathit{Tail}$ 
   \item $0.8 \ :: \ \tip(\mathit{Chimera}) \sqsubseteq \exists \mathit{has}.\mathit{Mane}$ 
   \item $0.9 \ :: \ \tip(\mathit{Chimera}) \sqsubseteq \exists \mathit{has}.\mathit{Horns}$ 
   \item $0.9 \ :: \ \tip(\mathit{Chimera}) \sqsubseteq \exists \mathit{has}.\mathit{Beards}$ 
   \item $\mathit{Chimera} \sqsubseteq  \mathit{LegendaryCreature}$
   \item $0.9  \ :: \ \tip(\mathit{Chimera}) \sqsubseteq  \mathit{FireBreathing}$ 
   \item $0.8 \ :: \ \tip(\mathit{Chimera}) \sqsubseteq  \mathit{Fly}$    
   \item $0.9 \ :: \ \tip(\mathit{Chimera}) \sqsubseteq  \mathit{Aggressive}$ 
   \item $\mathit{White} \sqcap \mathit{Yellow} \sqsubseteq \bot$
\end{itemize}

\section{Iterated Generation of Concepts: Combining Concepts from $C$-revised Knowledge Bases}\label{sez:iterativa}
A $C$-\emph{revised} knowledge base in the logic $\pfl$ is still in the language of the $\pfl$ logic. This allows us to iteratively repeat the same procedure in order to combine not only atomic concepts, but also compound concepts. 
In this section we show that our approach  can handle also the concept combination of $C$-revised knowledge bases.

Let us consider the compound concept obtained as the combination of an anti-hero -- as described by the $\mathit{AntiHero}$-knowledge base of Section \ref{sec:antieroe}\footnote{In this example we consider the first $\mathit{AntiHero}$-revised knowledge base obtained in Section \ref{sec:antieroe}.} -- and a chimera -- $\mathit{Chimera}$-revised knowledge base in Section \ref{sec:chimera}. 

The starting TBox $\TT$ is as follows:
\begin{enumerate}
   \item $\mathit{Chimera} \sqsubseteq  \mathit{Animal}$
   \item $0.9 \ :: \ \tip(\mathit{Chimera}) \sqsubseteq \exists \mathit{has}.\mathit{Tail}$ 
   \item $0.8 \ :: \ \tip(\mathit{Chimera}) \sqsubseteq \exists \mathit{has}.\mathit{Mane}$ 
   \item $0.9 \ :: \ \tip(\mathit{Chimera}) \sqsubseteq \exists \mathit{has}.\mathit{Horns}$ 
   \item $0.9 \ :: \ \tip(\mathit{Chimera}) \sqsubseteq \exists \mathit{has}.\mathit{Beards}$ 
   \item $\mathit{Chimera} \sqsubseteq  \mathit{LegendaryCreature}$
   \item $0.9  \ :: \ \tip(\mathit{Chimera}) \sqsubseteq  \mathit{FireBreathing}$ 
   \item $0.8 \ :: \ \tip(\mathit{Chimera}) \sqsubseteq  \mathit{Fly}$    
   \item $0.9 \ :: \ \tip(\mathit{Chimera}) \sqsubseteq  \mathit{Aggressive}$ 
        \item $\mathit{AntiHero} \sqsubseteq \exists \mathit{hasOpponent}.\mathit{Villain}$
        \item $\mathit{AntiHero} \sqsubseteq \exists \mathit{fightsFor}.\mathit{PersonalGoal}$
        \item $\mathit{AntiHero} \sqsubseteq \mathit{WithNegativeMoralValues}$
        \item 0.95 \ :: \ $\tip(\mathit{AntiHero}) \sqsubseteq \mathit{Protagonist}$
        \item 0.75 \ :: \ $\tip(\mathit{AntiHero}) \sqsubseteq \mathit{DemoniacIconicity}$
        \item 0.8 \ :: \ $\tip(\mathit{AntiHero}) \sqsubseteq  \mathit{Impulsive}$
\end{enumerate}

\noindent We discard the most obvious scenario having a probability of $20.4\%$ and considering all the typicality inclusions. Also the immediately lower scenario with probability $7.18\%$ is discarded since it inherits all the properties of the HEAD (and, therefore, is trivial).   As for the case of the villain chair of Section  \ref{sez:novel}, we firstly consider the following HEAD-MODIFIER combination: chimera is the HEAD, whereas anti-hero is the MODIFIER. Given this, the scenarios having the highest probability ($5.3\%$), giving preference to the HEAD, are the ones discarding either 
3 (has mane) or 8 (fly), and including all the other typicality assumptions.
These scenarios are those preferred from a cognitive point of view.

 \vspace{0.8cm}
 \hrule
 
\begin{itemize}
   \item $\mathit{Chimera} \sqsubseteq  \mathit{Animal}$
   \item $0.9 \ :: \ \tip(\mathit{AntiHero \sqcap Chimera}) \sqsubseteq \exists \mathit{has}.\mathit{Tail}$ 
   \item $0.9 \ :: \ \tip(\mathit{AntiHero \sqcap Chimera}) \sqsubseteq \exists \mathit{has}.\mathit{Horns}$ 
   \item $\mathit{Chimera} \sqsubseteq  \mathit{LegendaryCreature}$
   \item $0.9 \ :: \ \tip(\mathit{AntiHero \sqcap Chimera}) \sqsubseteq  \mathit{Aggressive}$ 
        \item $\mathit{AntiHero} \sqsubseteq \exists \mathit{hasOpponent}.\mathit{Villain}$
        \item $\mathit{AntiHero} \sqsubseteq \exists \mathit{fightsFor}.\mathit{PersonalGoal}$
        \item $\mathit{AntiHero} \sqsubseteq \mathit{WithNegativeMoralValues}$
        \item 0.75 \ :: \ $\tip(\mathit{AntiHero \sqcap Chimera}) \sqsubseteq \mathit{Protagonist}$
        \item 0.9 \ :: \ $\tip(\mathit{AntiHero \sqcap Chimera}) \sqsubseteq \mathit{DemoniacIconicity}$
        \item 0.8 \ :: \ $\tip(\mathit{AntiHero \sqcap Chimera}) \sqsubseteq  \mathit{Impulsive}$
\end{itemize}

\hrule 

\begin{itemize}
   \item $\mathit{Chimera} \sqsubseteq  \mathit{Animal}$
   \item $0.9 \ :: \ \tip(\mathit{AntiHero \sqcap Chimera}) \sqsubseteq \exists \mathit{has}.\mathit{Beards}$ 
   \item $\mathit{Chimera} \sqsubseteq  \mathit{LegendaryCreature}$
   \item $0.9  \ :: \ \tip(\mathit{AntiHero \sqcap Chimera}) \sqsubseteq  \mathit{FireBreathing}$ 
   \item $0.9 \ :: \ \tip(\mathit{AntiHero \sqcap Chimera}) \sqsubseteq  \mathit{Aggressive}$ 
        \item $\mathit{AntiHero} \sqsubseteq \exists \mathit{hasOpponent}.\mathit{Villain}$
        \item $\mathit{AntiHero} \sqsubseteq \exists \mathit{fightsFor}.\mathit{PersonalGoal}$
        \item $\mathit{AntiHero \sqcap Chimera} \sqsubseteq \mathit{WithNegativeMoralValues}$
        \item 0.75 \ :: \ $\tip(\mathit{AntiHero \sqcap Chimera}) \sqsubseteq \mathit{Protagonist}$
        \item 0.9 \ :: \ $\tip(\mathit{AntiHero \sqcap Chimera}) \sqsubseteq \mathit{DemoniacIconicity}$
        \item 0.8 \ :: \ $\tip(\mathit{AntiHero \sqcap Chimera}) \sqsubseteq  \mathit{Impulsive}$
\end{itemize}

\hrule

\vspace{0.8cm}

 We could also consider more surprising scenarios by tuning both the number of properties inherited by the compound concepts and by reverting the assignment of the HEAD-MODIFIER pair among the concepts.
 In the first case, let us restrict our concern to scenarios allowing the chimera anti-hero to inherit 6 properties\footnote{Of course, the number of properties can be considered as a parameter through which it is possible to play with the mechanisms underlying the logic $\pfl$.} from the original concepts. In this case, the scenario selected according to $\pfl$ has probability $0.11\%$ and includes 5 out of 7 inclusions of the HEAD, namely 2, 4, 5, 7, 9, and only one of the MODIFIER, namely 13. In this case, the $\mathit{AntiHero}-\mathit{Chimera}$-revised knowledge base is as follows:
 
\begin{itemize}
   \item $\mathit{Chimera} \sqsubseteq  \mathit{Animal}$
   \item $0.9 \ :: \ \tip(\mathit{AntiHero} \sqcap \mathit{Chimera}) \sqsubseteq \exists \mathit{has}.\mathit{Tail}$ 
   \item $0.9 \ :: \ \tip(\mathit{AntiHero} \sqcap \mathit{Chimera}) \sqsubseteq \exists \mathit{has}.\mathit{Horns}$ 
   \item $0.9 \ :: \ \tip(\mathit{AntiHero} \sqcap \mathit{Chimera}) \sqsubseteq \exists \mathit{has}.\mathit{Beards}$ 
   \item $\mathit{Chimera} \sqsubseteq  \mathit{LegendaryCreature}$
   \item $0.9  \ :: \ \tip(\mathit{AntiHero} \sqcap \mathit{Chimera}) \sqsubseteq  \mathit{FireBreathing}$ 
   \item $0.9 \ :: \ \tip(\mathit{AntiHero} \sqcap \mathit{Chimera}) \sqsubseteq  \mathit{Aggressive}$ 
        \item $\mathit{AntiHero} \sqsubseteq \exists \mathit{hasOpponent}.\mathit{Villain}$
        \item $\mathit{AntiHero} \sqsubseteq \exists \mathit{fightsFor}.\mathit{PersonalGoal}$
        \item $\mathit{AntiHero} \sqsubseteq \mathit{WithNegativeMoralValues}$
        \item 0.95 \ :: \ $\tip(\mathit{AntiHero} \sqcap \mathit{Chimera}) \sqsubseteq \mathit{Protagonist}$
 \end{itemize} 
 
 Notice that, as in all the previous examples, the combined concept implicitly  inherits all the rigid properties of both HEAD and MODIFIER: here, for instance, we derive that $\mathit{AntiHero} \sqcap \mathit{Chimera} \sqsubseteq  \mathit{Animal}$, as well as $\mathit{AntiHero} \sqcap \mathit{Chimera} \sqsubseteq \exists \mathit{hasOpponent}.\mathit{Villain}$, by the rigid properties $\mathit{Chimera} \sqsubseteq  \mathit{Animal}$ and $\mathit{AntiHero} \sqsubseteq \exists \mathit{hasOpponent}.\mathit{Villain}$, respectively.

Let us conclude this example by swapping the HEAD and the MODIFIER, namely we exploit the logic $\pfl$ in order to define a anti-hero chimera by considering the concept $\mathit{AntiHero}$ as the HEAD and the concept $\mathit{Chimera}$ as the MODIFIER. This is just based on the assumption that the role played by the novel character can be more relevant from a narrative point of view. 
Again, we discard the most obvious scenario having a probability of $20.4\%$ and considering all the typicality inclusions. However, the scenario with immediately lower probability ($7.18\%$) is selected, since it inherits only 2 out of 3 properties from the HEAD, namely those corresponding to inclusions 13 and 14. The resulting $\mathit{AntiHero-Chimera}$-revised knowledge base is as follows: 

\begin{itemize}
   \item $\mathit{Chimera} \sqsubseteq  \mathit{Animal}$
   \item $0.9 \ :: \ \tip(\mathit{AntiHero \sqcap Chimera}) \sqsubseteq \exists \mathit{has}.\mathit{Tail}$ 
   \item $0.8 \ :: \ \tip(\mathit{AntiHero \sqcap Chimera}) \sqsubseteq \exists \mathit{has}.\mathit{Mane}$ 
   \item $0.9 \ :: \ \tip(\mathit{AntiHero \sqcap Chimera}) \sqsubseteq \exists \mathit{has}.\mathit{Horns}$ 
   \item $0.9 \ :: \ \tip(\mathit{AntiHero \sqcap Chimera}) \sqsubseteq \exists \mathit{has}.\mathit{Beards}$ 
   \item $\mathit{Chimera} \sqsubseteq  \mathit{LegendaryCreature}$
   \item $0.9  \ :: \ \tip(\mathit{AntiHero \sqcap Chimera}) \sqsubseteq  \mathit{FireBreathing}$ 
   \item $0.8 \ :: \ \tip(\mathit{AntiHero \sqcap Chimera}) \sqsubseteq  \mathit{Fly}$    
   \item $0.9 \ :: \ \tip(\mathit{AntiHero \sqcap Chimera}) \sqsubseteq  \mathit{Aggressive}$ 
        \item $\mathit{AntiHero} \sqsubseteq \exists \mathit{hasOpponent}.\mathit{Villain}$
        \item $\mathit{AntiHero} \sqsubseteq \exists \mathit{fightsFor}.\mathit{PersonalGoal}$
        \item $\mathit{AntiHero} \sqsubseteq \mathit{WithNegativeMoralValues}$
        \item 0.95 \ :: \ $\tip(\mathit{AntiHero \sqcap Chimera}) \sqsubseteq \mathit{Protagonist}$
        \item 0.8 \ :: \ $\tip(\mathit{AntiHero \sqcap Chimera}) \sqsubseteq  \mathit{Impulsive}$
\end{itemize}



\section{Conclusions, Related Works and Future Research}\label{sez:conclusioni}
We have introduced a nonmonotonic Description Logic $\pfl$ for concept combination, extending the DL of typicality $\alct$ with a DISPONTE semantics and with a blocking inheritance selection heuristics coming from the cognitive semantics. This logic enjoys good computational properties, since entailment in it remains \texttt{ExpTime} as the underlying monotonic $\alc$, and is able to take into account the concept combination of prototypical properties.
To this aim, the logic $\pfl$ allows to have inclusions of the form $p \ :: \ \tip(C) \sqsubseteq D$, representing that, with a probability $p$, typical $C$s are also $D$s. Then, several different scenarios -- having different probabilities -- are described by including or not such inclusions, and prototypical properties of combinations of concepts are obtained by restricting reasoning services to scenarios having suitable probabilities, excluding ``trivial'' ones with the highest probabilities.

Several approaches in extending DLs with nonmonotonic capabilities  have been proposed in the
literature. All these approaches are essentially based on the integration of DLs  with  well established nonmonotonic reasoning mechanisms    [\cite{bonattilutz2,baader95b,donini2002,AIJ,Casinistraccia2010,CasiniStracciaJAIR,bonattisauropetrova}], ranging from Reiter's defaults to minimal knowledge and negation as failure.  
In [\cite{lukas}]  two probabilistic extensions of  Description Logics $\mathcal{SHIF}$({\bf D}) and $\mathcal{SHOIN}$({\bf D}) are introduced. These extensions are semantically based on the notion of probabilistic lexicographic entailment [\cite{lex}] and allow to represent and reason about prototypical properties of classes that are semantically interpreted as lexicographic entailment introduced by Lehmann from conditional knowledge bases. Intuitively, the basic idea is to interpret inclusions of the TBox and facts in the ABox as probabilistic knowledge about random and concrete instances of concepts. As an example, in these extensions one can express that ``typically, a randomly chosen student makes use of social networks with a probability of at least $70\%$'' with a formula of the form $(\mathit{SocialNetworkUser} \ \mid \ \mathit{Student})[0.7, \ 1]$. As the logic of typicality $\alct$ underlying our work, the lexicographic entailment defined in [\cite{lukas}] inherits interesting and useful nonmonotonic properties from lexicographic entailment in [\cite{lex}], such as specificity, rational monotonicity and some forms of irrelevance. As a difference with our approach, here probabilities are used in order to capture an alternative notion of typicality of concepts, rather than to express degrees of belief of typicality inclusions that are at the base of our notion of our machinery for concept combination. All these approaches could be seen as possible alternative with respect to the logic of typicality $\alct$, representing one of the building blocks of the logic $\pfl$. However, we do not investigate here to what extent these approaches can actually be compliant with respect to the problem of handling typicality based concept combinations. To the best of our knowledge, there are not existing works exploiting these logics to handle the problem in focus and we believe this aspect is worth-considering for future works. 

Another related work with respect to the proposed formalism is in [\cite{ecsqaru2017}]. Here   
the author extends the logic of typicality $\alct$ by means of probabilities equipping typicality inclusions of the form $\tip(C) \sqsubseteq_p D$, whose intuitive meaning is that, ``normally, $C$s are $D$s and we have a probability of $1-p$ of having exceptional $C$s not being $D$s''. Probabilities of exceptions are then used in order to reason about plausible scenarios, obtained by selecting only {\em some}  typicality assumptions and whose probabilities belong to a given and fixed range. As a difference with the logic $\pfl$, all typicality assumptions are systematically taken into account: as a consequence, one cannot exploit such a DL for capturing compositionality, since it is not possible to block  inheritance of  prototypical properties in concept combination. The logic $\pfl$ extends the work of [\cite{ecsqaru2017}] in that it does not  systematically take into account all typicality assumptions. As a consequence, $\pfl$ allows to block inheritance of  prototypical properties in concept combination.
The same criticism applies also to the approach proposed in [\cite{IA2017}], where $\alct$ is extended by inclusions of the form $\tip(C) \sqsubseteq_d D$, where $d$ is a \emph{degree of expectedness}, used to define a preference relation among  extended ABoxes: entailment of queries is then restricted to ABoxes that are minimal with respect to such preference relations and that represent surprising scenarios. Also in this case, however, the resulting logic does not allow to define scenarios containing only some inclusions, since all of them are systematically considered.
Similarly, probabilistic DLs [\cite{riguzzi}] themselves cannot be employed as a framework for dealing with the combination of prototypical concepts, since these logics are not equipped with reasoning mechanisms needed for typicality based reasoning (and that, in $\pfl$, are inherited by $\alct$).

Several approaches have been recently introduced  in order to tackle the problem of reasoning under probabilistic uncertainty in Description Logics. 
In [\cite{lukasstraccia}] the authors combine fuzzy Description Logics, fuzzy logic programs and probabilistic uncertainty. An extension of the lightweight Description Logic $\mathit{DL-Lite}$ within a possibility theory setting is proposed in [\cite{benf}]. In this approach, uncertainty degrees are associated to inclusions in order to define the inconsistency degree of a KB.  All these approaches neglect to consider the proposed frameworks as the basis for the combination of concepts since,
as already mentioned, approaches based on fuzziness fail to do the job.

For what concerns, more specifically, the modelling of  prototypical concept composition in a human-like fashion (and with human-level performances), several approaches have been proposed in both the AI and computational cognitive science communities. Lewis and Lawry [\cite{lewis2016hierarchical}] present a detailed analysis of the limits of the set-theoretic approaches [\cite{montague1973proper}], the fuzzy logics [\cite{zadeh1975concept,dubois1997three}] (whose limitations was already shown in [\cite{osherson1981adequacy,smith1984conceptual,hampton2011conceptual}]), the vector-space models [\cite{mitchell2010composition}] and quantum probability approaches [\cite{aerts2013concepts}] proposed to model this phenomenon. In addition, they propose to use hierarchical conceptual spaces [\cite{gardenfors2014geometry}] to model the phenomenon in a way that accurately reflects how humans exploit their creativity in conjunctive concept combination. While we agree with the authors with the comments moved to the described approaches, in this work we have shown that our logic can equally model, in a cognitively compliant-way, the composition of prototypes by using a computationally effective nonmonotonic formalism. In particular, our model is able to meet all following cognitive requirements [\cite{lewis2016hierarchical,hampton2011conceptual}]: i) it provides a blocking mechanism of property inheritance for prototypical concept combination thus enabling the possibility of dealing with a non-standard compositional behavior ii) it is able to deal with the phenomenon of attribute emergence (and loss) for the combined concept iii) it preservers the notions of necessity and impossibility of property attribution for the combined concept iv) it explicitly assumes that the combination is not commutative (i.e. the different attribution of the HEAD-MODIFER roles does non provide the same combined concept) and that v) there are dominance effects in the concepts to be combined (both these effects are obtained via the HEAD-MODIFIER heuristics).

With respect to other formal approaches developed to model the same phenomenon [\cite{smith1988combining,kamp1995prototype}], we have also shown that our formalism is able to account for forms of NOUN-NOUN concept combination (e.g. like stone lion or, similarly, like porcelain cat) that such frameworks are not able to model (see [\cite{gardenfors1998concept}]). Also: our framework does not provide any increase in the reasoning complexity with respect to the standard monotonic DL $\alc$.

It is also worth-noticing that the overall reasoning procedure of  $\pfl$ is directly inspired to the Composite Prototype Model (CPM) proposed by Hampton in order to account for the phenomenon of typicality-based composition [\cite{hampton1987inheritance,hampton1988overextension,hampton2017compositionality}]. 
The similarity to CPM concerns both the identification of dominant conceptual features to consider for the composition (and this is reflected, in  $\pfl$ , by considering: i) the distinction between typical and rigid properties ii) the probability values associated to every typical inclusions iii) the HEAD-MODIFIER difference) and the so called phenomenon of \emph{attribute emergence} (shown in the \emph{Pet Fish} case).
As a difference with respect to the Hampton's model, the reasoning procedure of  $\pfl$  does not assume a process proceeding, first, on the attempt of obtaining the compound by composing the features and then via a process of successive identification and amalgamation of the incompatibility eventually occurring. In our case, in fact, the logic $\pfl$ obtains the combination by directly excluding the inconsistent scenarios and resorts to all its constituent ingredients (i.e. the non monotonic procedure of $\alct$,  the probabilistic ranking provided by the DISPONTE semantics and the HEAD-MODIFIER heuristics) to determine which features to inherit.

Other attempts similar to the one proposed here concerns the modelling of the conceptual blending phenomenon: a task where the obtained concept is \emph{entirely novel} and has no strong association with the two base concepts (for details about the differences between conceptual combination and conceptual blending see [\cite{nagai2006formaldescription}]). 
In this setting, [\cite{confalonieri2016}] proposed a mechanism for conceptual blending based on the DL $\mathcal{EL}^{++}$. They construct the generic space of two  concepts by introducing an upward refinement operator that is used for finding common generalizations of $\mathcal{EL}^{++}$ concepts. However, differently from us, what they call prototypes are expressed in the standard monotonic DL, which does not allow to reason about typicality and defeasible inheritance. More recently, a different approach is proposed in [\cite{eppe2018computational}], where the authors see the problem of concept blending as a nonmonotonic search problem and proposed to use Answer Set Programming (ASP) to deal with this search problem. As we have shown in the Chimera example, $\pfl$ is flexible enough to be applied also to the case of conceptual blending. There is no evidence, however, that both the frameworks of [\cite{confalonieri2016}] and [\cite{eppe2018computational}] would be able to model (in toto or in part) conceptual combination problems like the PET FISH. As such, $\pfl$ seems to provide a more general mechanism for modelling the combinatorial phenomenon of concept invention (that can be obtained both with combination and blending). 

We are currently developing an efficient reasoner for the logic $\pfl$, relying on the prover RAT-OWL [\cite{ratowl}] for reasoning in the nonmonotonic logic $\alct$ underlying our approach and on the well established HermiT reasoner. The first version of the system
is implemented in Pyhton and exploits a translation of an $\alct$ knowledge base into standard $\alc$.

In future research we aim at extending our approach to more expressive DLs, such as those underlying the standard OWL language. Starting from the work of [\cite{DL2014}], applying the logic with the typicality operator and the rational closure to $\mathcal{SHIQ}$, we intend to study whether and how $\pfl$ could provide an alternative solution to the problem of the ``all or nothing'' behavior of rational closure with respect to property inheritance. 

We envision different areas of application for our framework. We have already mentioned the field of computational creativity, but other employments can be considered. For example, we plan to use the logic $\pfl$ as the basis for the exploitation of a new area of autonomic computing [\cite{auton}] concerning the problem of the automatic generation of novel knowledge in a cognitive artificial agent, starting from an initial commonsense knowledge base. This approach will require to enrich the knowledge processing mechanisms of current cognitive agents [\cite{lieto2017knowledge}] and presents an element of innovation in that it does not assume that the only way to process and reason on new knowledge is via an external injection of new information but via a process of automatic knowledge generation. We believe that such an approach can find practical applications in the areas of cognitive agent architectures and robotics. Finally, as a short-term future work, we aim at investigating the use of $\pfl$ for generating metaphors based on frame-semantics (and therefore starting from event-centric representations instead of subject-centric ones) by using the recently developed repository Amnestic Forgery [\cite{DBLP:conf-fois-GangemiAP18}]. Should $\pfl$ prove to be able to model also metaphorical phenomena starting from a different semantic approach, this would represent an additional symptom that its underlying  procedure captures some foundational elements of commonsense concept combination and invention.

%
%

\bibliographystyle{apalike}
\bibliography{ijcai18.bib}

\end{document}